\documentclass{article}
\usepackage{fullpage}

\usepackage{authblk} 
\usepackage{hyperref,url}
\usepackage{graphicx,subcaption}
\usepackage{amsmath,amssymb,amsthm}
\usepackage{cleveref}

\graphicspath{{./figures/}}

\def\Var{\mathrm{Var}}

\def\bbR{\mathbb{R}}

\def\dx{\mathrm{d}x}

\def\eqdef{:=}
\def\bbR{\mathbb{R}}
\DeclareMathOperator{\sigm}{\mathrm{sigm}}
\DeclareMathOperator{\relu}{\mathrm{relu}}
\DeclareMathOperator{\elu}{\mathrm{elu}}
\DeclareMathOperator{\softplus}{\mathrm{softplus}}
\newlength{\myfigwidth}
\newtheorem{theorem}{Theorem}
\newtheorem{lemma}[theorem]{Lemma}
\newtheorem{corollary}[theorem]{Corollary}
\newtheorem{proposition}[theorem]{Proposition}

\title{$q$-Neurons: Neuron Activations based on Stochastic Jackson's Derivative Operators}

\author[$\star$]{Frank Nielsen}
\affil[$\star$]{Sony Computer Science Laboratories, Inc.}
\affil[$\star$]{Japan}
\affil[$\star$]{{\small\tt Frank.Nielsen@acm.org}}
\author[$\dagger$]{Ke Sun}
\affil[$\dagger$]{Data61}
\affil[$\dagger$]{Australia}
\affil[$\dagger$]{{\small\tt Ke.Sun@data61.csiro.au}}
\date{}

\begin{document}

\maketitle

\begin{abstract}
We propose a new generic type of stochastic neurons, called $q$-neurons, that considers activation functions based on Jackson's $q$-derivatives, with stochastic parameters $q$.
Our generalization of neural network architectures with $q$-neurons is shown to be both scalable and very easy to implement.
We demonstrate experimentally consistently improved performances over state-of-the-art standard activation functions, both on training and testing loss functions.
\end{abstract}

\section{Introduction}

The vanilla method to train a {\em Deep Neural Network} (DNN) is to use the {\em Stochastic Gradient Descent} (SGD) method
(a first-order local optimization technique). The gradient of the DNN loss function, represented as a directed
computational graph, is calculated using the efficient backpropagation algorithm relying on the chain rule of
derivatives (a particular case of automatic differentiation).

The ordinary derivative calculus can be encompassed into a more general {\em $q$-calculus}~\cite{Jackson-1909,qbook-2001} by defining the Jackson's {\em $q$-derivative} (and gradient) as follows:
\begin{equation}
D_{q} f(x) \eqdef    \frac{f(x)-f(qx)}{(1-q)x},\quad q\neq1, x\neq0.
\end{equation}

	The $q$-calculus generalizes the ordinary Leibniz gradient (obtained as a limit case when $q\rightarrow 1$ or when $x\rightarrow 0$) but does  {\em not} enjoy a generic chain rule property.
It can  further be extended to the {\em $(p,q)$-derivative}~\cite{pq-2013,khan18} defined as follows:
\begin{equation}
D_{p,q} f(x) \eqdef  \frac{f(px)-f(qx)}{(p-q)x}=D_{q,p} f(x) ,\quad p\neq{q}, x\neq0.
\end{equation}
which encompasses the $q$-gradient as $D_{1,q}f(x) = D_{q,1}f(x) = D_{q}f(x)$.

The two main advantages of $q$-calculus are
\begin{enumerate}
\item To bypass the calculations of limits, and
\item To consider $q$ as a stochastic parameter.
\end{enumerate}

We refer to the textbook~\cite{qbook-2001} for an in-depth explanation of $q$-calculus.
Appendix~\ref{sec:pqgradient} recalls the basic rules and properties of the generic
$(p,q)$-calculus~\cite{khan18} that further generalizes the $q$-calculus.

To the best of our knowledge, the {\em $q$-derivative operators} have seldom been considered in the machine learning community~\cite{xn-2018}.
We refer to~\cite{qGradOpt-2016} for some encouraging preliminary experimental optimization results on global optimization tasks.
 
In this paper, we introduce a meta-family of neuron activation functions based on standard activation functions (e.g., sigmoid, softplus, ReLU, ELU).
We refer to them as {\em $q$-activations}.
The $q$-activation is a {\em stochastic
activation function} built  on top of any given activation function $f$.
$q$-Activation is very easy to implement based on state-of-the-art Auto-Differentiation (AD)
frameworks while consistently producing better performance. Based on our
experiments, one should almost always use $q$-activation instead of its
deterministic counterpart. 
In the remainder, we define {\em $q$-neurons} as stochastic neurons equipped with $q$-activations.

Our main contributions are summarized as follows:
\begin{itemize}
\item The generic $q$-activation and an analysis of its basic properties.
\item An empirical study that demonstrates that the $q$-activation can reduce both training and testing errors.
\item A novel connection and sound application of (stochastic) $q$-calculus in machine learning.
\end{itemize}

\section{Neurons with $q$-activation functions}

Given {\em any} activation function $f:\,\bbR\to\bbR$, we construct its corresponding
``quantum'' version, also called \emph{$q$-activation function}, as
\begin{equation}\label{eq:qneuron}
g_q(x) \eqdef \frac{f(x)-f(qx)}{1-q} = \left(D_qf(x)\right) x,
\end{equation}
where $q$ is a real-valued random variable.
To see the relationship between $g_q(x)$ and $f(x)$, let us observe that we have the following asymptotic properties:

\begin{proposition}\label{thm:limit}
Assume $f(x)$ is smooth and the expectation of $q$ is $E(q)=1$. Then $\forall{x}$, we have
\begin{align}
\lim_{\mathrm{Var}(q)\to0}  g_q(x) &=
f'(x) x.\nonumber\\
\lim_{\mathrm{Var}(q)\to0} g'_q(x) &=
f'(x) + f''(x) x,\nonumber
\end{align}
where $E(\cdot)$ denotes the expectation, and $\mathrm{Var}(\cdot)$ denotes the variance.
\end{proposition}
\begin{proof}
\begin{equation*}
\lim_{\mathrm{Var}(q)\to0}  g_q(x)
= \lim_{\mathrm{Var}(q)\to0}  \frac{f(x)-f(qx)}{x-qx} x
= \lim_{qx\to{x}}  \frac{f(x)-f(qx)}{x-qx} x
= f'(x)x.
\end{equation*}
\begin{align*}
&\lim_{\mathrm{Var}(q)\to0} g'_q(x)
= \lim_{\mathrm{Var}(q)\to0} \frac{f'(x)-qf'(qx)}{1-q}
= \lim_{\mathrm{Var}(q)\to0} \frac{f'(x)-q f'(x) + qf'(x) - qf'(qx)}{1-q}\\
&=f'(x) + \lim_{\mathrm{Var}(q)\to0} \frac{qf'(x) - qf'(qx)}{1-q}
= f'(x) + \lim_{q\to1} \frac{f'(x) - f'(qx)}{x-qx} qx
= f'(x) + f''(x) x.
\end{align*}
\end{proof}
Notice that as $\Var(q)\to 0$, the limit of $g_q(x)$ is {\em not} $f(x)$ but $f'(x)x$.
Thus informally speaking, the gradient of $g_q(x)$ carries {\em second-order information} of $f(x)$.

We further have the following property:

\begin{proposition}\label{thm:dpq}
We have for $p,q\not =1$:
\begin{align}\label{eq:dpq}
D_p(g_q(x)) = \frac{1}{1-p} D_q f(x) - \frac{p}{1-p} D_{p,pq} f(x).
\end{align}
\end{proposition}

\begin{proof}
\begin{align}
D_p(g_q(x))
&=
\frac{g_q(x)-g_q(px)}{(1-p)x}
=
\frac{ \left(D_qf(x)\right) x - \left(D_qf(px)\right)px}{(1-p)x}\nonumber\\
&=
\frac{1}{1-p}D_qf(x) -\frac{p}{1-p}D_qf(px).
\end{align}
Since $D_q f(px)=\frac{f(px)-f(pqx)}{px-pqx}=
\frac{f(px)-f(pqx)}{(p-pq)x}=D_{p,pq} f(x)$,
\cref{eq:dpq} is straightforward.
\end{proof}
By \cref{thm:dpq}, the $p$-derivative of 
the $q$-activation $g_q(x)$ agrees with the original activation function $f$.

See \cref{tbl:act} for a list of activation functions with their corresponding functions $f'(x)x$,
where 
$$
\sigm(x)=1/(1+\exp(-x)),
$$ 
is the {\em sigmoid function},
$$
\softplus(x)=\log(1+\exp(x)),
$$
is the {\em softplus function},
\begin{equation*}
\relu(x)=\left\{
\begin{array}{ll}
x & \text{ if }x\ge0\\
0 & \text{otherwise}
\end{array}
\right.
\end{equation*}
is the {\em Rectified Linear Unit} (ReLU)~\cite{maas13},
and
\begin{equation*}
\elu(x)=\left\{
\begin{array}{ll}
x & \text{ if }x\ge0\\
\alpha(\exp(x)-1) & \text{otherwise}
\end{array}
\right.
\end{equation*}
denotes the {\em Exponential Linear Unit} (ELU)~\cite{clevert15}.

\begin{table*}
\centering
\caption{Common activation functions $f(x)$ with their corresponding limit cases $\lim_{\mathrm{Var}(q)\to0}  g_q(x)=f'(x)x$.}\label{tbl:act}
\vspace{.5em}
\begin{tabular}{|l||c|c|c|c|c|}\hline
$f(x)$   & $\sigm(x)$ & $\tanh(x)$ & $\relu(x)$ & $\softplus(x)$& $\elu(x)$ \\ \hline
$f'(x)x$ & $\sigm(x)(1-\sigm(x))x$ & $\mathrm{sech}^2(x)x$ & $\mathrm{relu}(x)$ & $\sigm(x)x$ & $\left\{ \begin{array}{ll}
               x & x\ge0\\
\alpha \exp(x) x & x<0 \end{array}\right.$\\ \hline
\end{tabular}
\end{table*}

A common choice for the random variable $q$ that is used in our experiments is
\begin{equation}\label{eq:q}
q = 1 + \left(2[\epsilon\ge0]-1\right) \left(\lambda \vert\epsilon\vert + \phi\right),
\end{equation}
where $\epsilon\sim N(0,1)$ follows the standard Gaussian distribution, 
$[\cdot]$ denotes the Iverson bracket  (meaning $1$ if the proposition is satisfied, and $0$ otherwise),
$\lambda>0$ is a scale parameter of $q$, and $\phi=10^{-3}$
is the smallest absolute value of $q$ so as to avoid division by zero.
See \cref{fig:q} for the density function plots of $q$ defined on $(-\infty,-\phi]\cup[\phi,\infty]$

\begin{figure} 
\centering
\includegraphics[width=.9\textwidth]{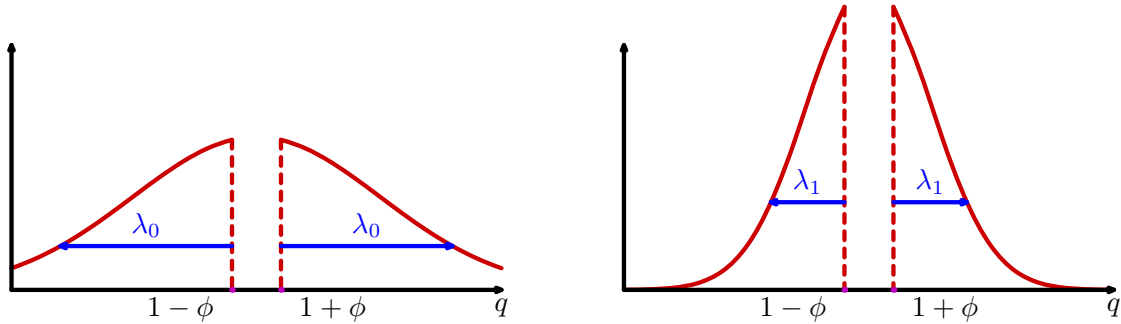}
\caption{The probability density function of stochastic variable $q$ (zero-avoiding) used when calculating $q$-derivatives.}\label{fig:q}
\end{figure}
\begin{figure*}[!ht]
\centering
\begin{subfigure}[b]{.9\textwidth}
\includegraphics[width=\textwidth]{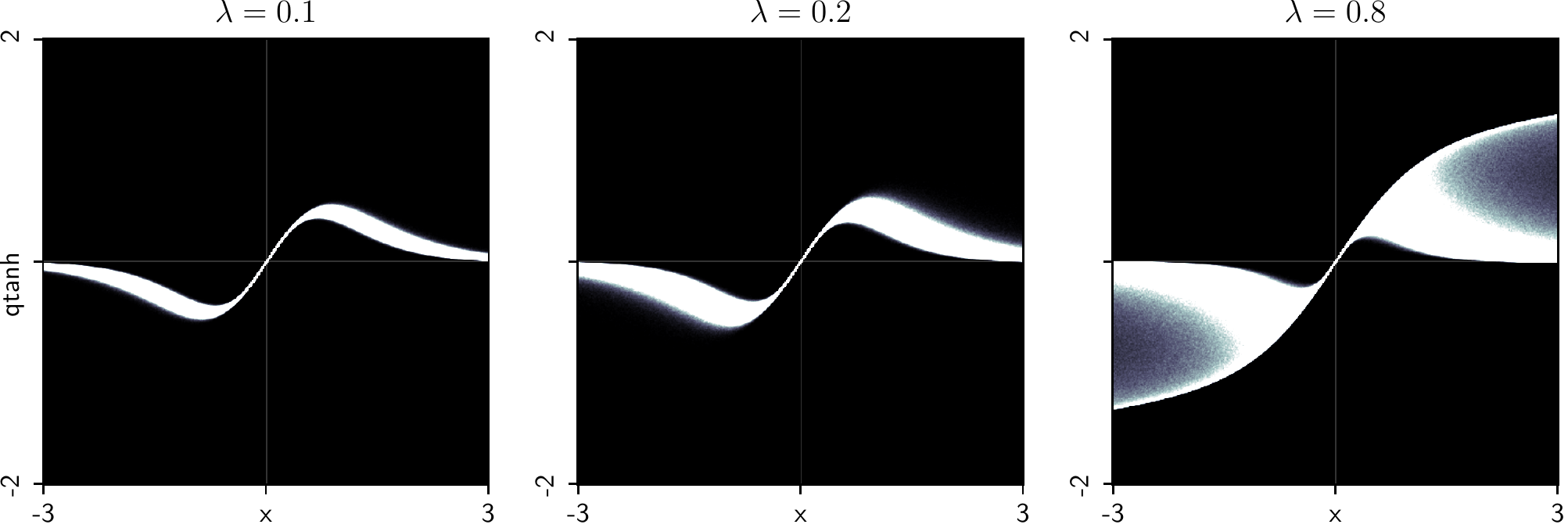}
\end{subfigure}
\begin{subfigure}[b]{.9\textwidth}
\includegraphics[width=\textwidth]{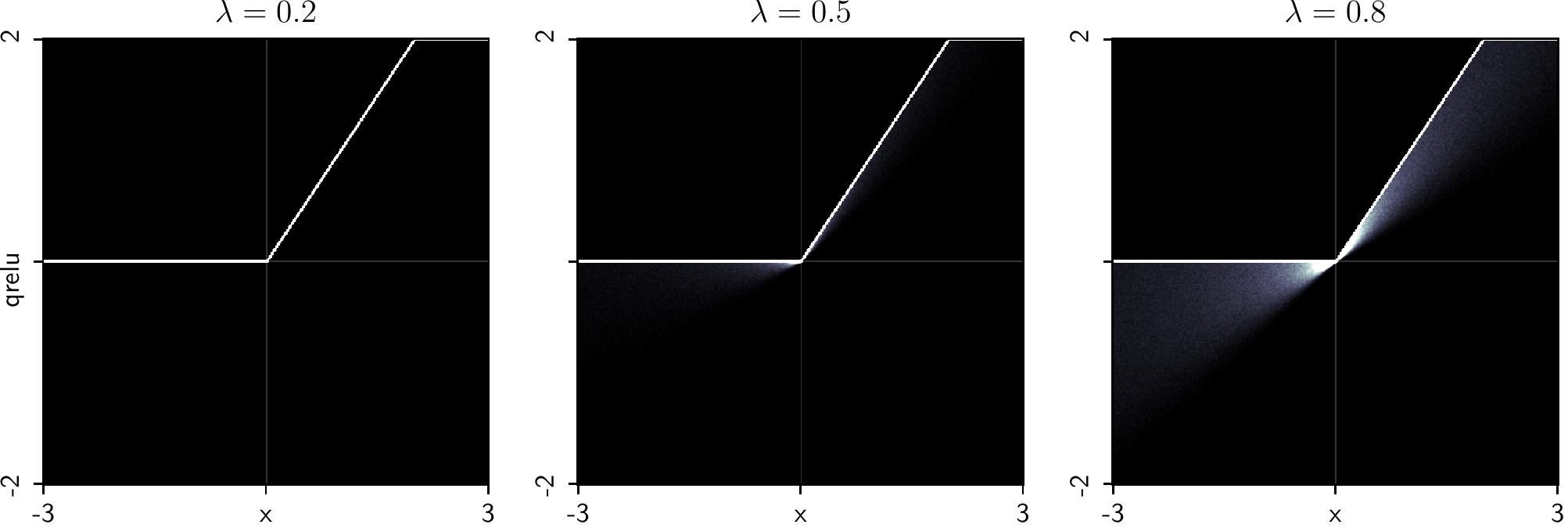}
\end{subfigure}
\begin{subfigure}[b]{.9\textwidth}
\includegraphics[width=\textwidth]{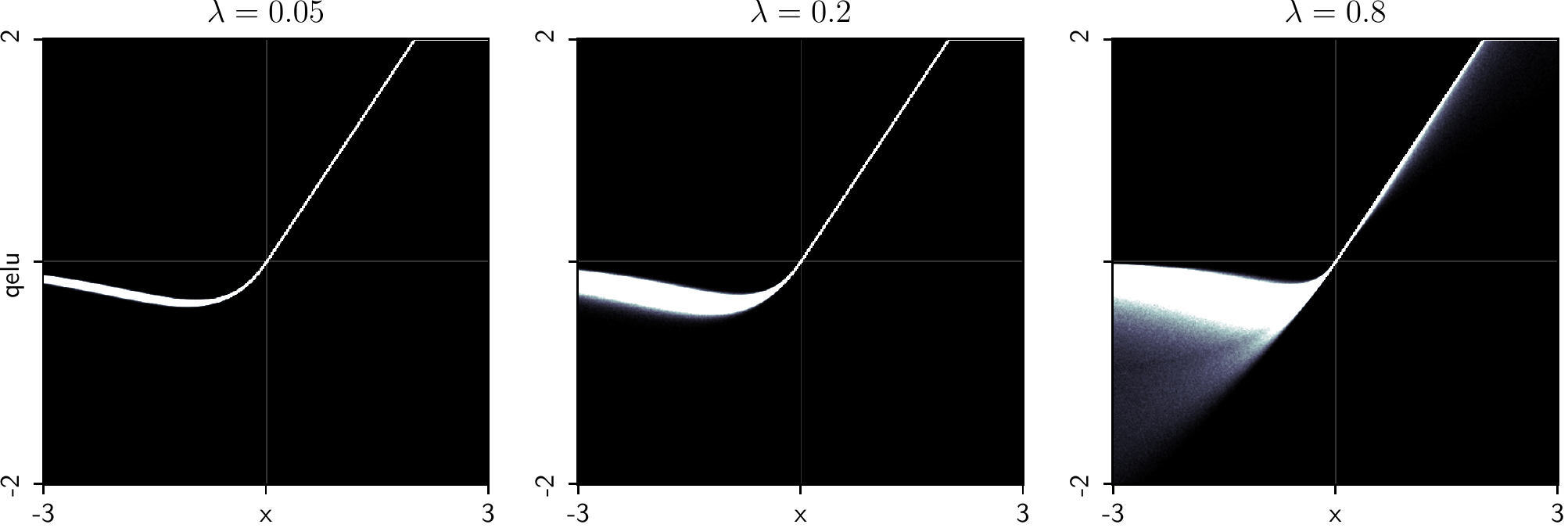}
\end{subfigure}
\begin{subfigure}[b]{.9\textwidth}
\includegraphics[width=\textwidth]{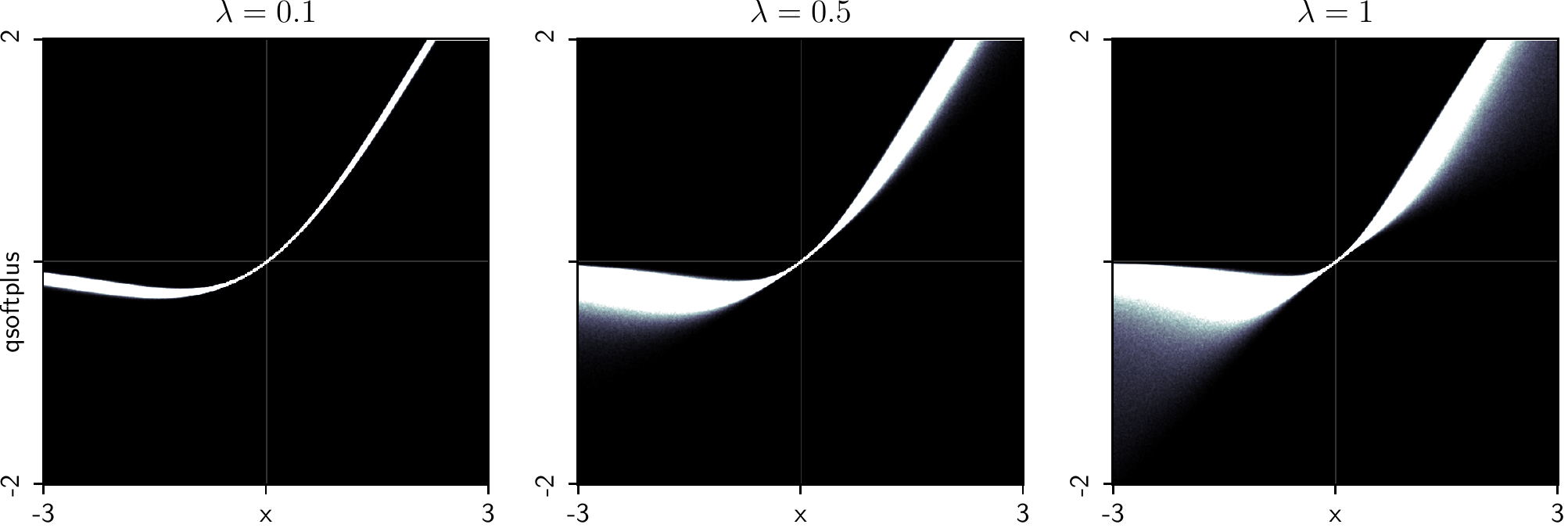}
\end{subfigure}
\caption{The density function of $q$-neurons with 
$q$ sampled according to \cref{eq:q} for different
values of $\lambda$.
The activation is roughly a deterministic function $f'(x)x$ for small $\lambda$ as shown in \cref{tbl:act}.
The activation is random for large $\lambda$.
Darker color indicates higher probability density.}
\end{figure*}

To implement $q$-neurons, one only need to tune the hyper-parameter $\lambda$.
It can either be fixed to a small value, e.g. 0.02 or 0.05 during learning,
or be annealed from an initial value $\lambda_0$.
Such an annealing scheme can be set to
\begin{equation}\label{eq:anneal}
\lambda = \frac{\lambda_0}{1+\gamma (T-1)},
\end{equation}
where $T=1,2,\cdots$ is the index of the current epoch, and $\gamma$ is a decaying rate parameter.
This parameter $\gamma$ can be empirically fixed based on the total number of epochs: 
For example, in our experiments we train $100$ epochs and apply $\gamma=0.5$,
so that in the final epochs $\lambda$ is a same value (around 0.02$\lambda_0$).
We will investigate both of those two cases in our experiments.

Let us stress out that deep learning architectures based on stochastic $q$-neurons are scalable
and easy to implement.  There is {\em no} additional free parameter imposed.
The computational overhead of $g_q(x)$ as compared to $f(x)$
involves sampling one Gaussian random variable, and then 
calling $f(x)$ two times and computing $g_q(x)$ according to \cref{eq:qneuron}.
In our \texttt{Python} implementation, the core implementation of $q$-neuron 
is only in three lines of codes (see~\ref{sec:pseudocode}).

Alternative approaches to inject stochasticity into neural network training
include dropout~\cite{dropout}, gradient noise~\cite{gradientnoise}, etc.
Both $q$-neuron and dropout modify the forward pass of the neural network.
In the experimental section we will investigate the {\em effect} of $q$-neurons
with or without dropout.

\section{Experiments}\label{sec:exp}

We carried experiments on classifying MNIST digits\footnote{\url{http://yann.lecun.com/exdb/mnist/}}
and CIFAR10 images\footnote{\url{https://www.cs.toronto.edu/~kriz/cifar.html}}
using Convolutional Neural Networks (CNNs) and Multi-Layer Perceptrons (MLPs).
Our purpose is not to beat state-of-the-art records but to
investigate the effect of applying $q$-neuron and its hyper-parameter sensitivity.
We summarize the neural network architectures as follows:

\begin{itemize}
\item The MNIST-CNN architecture is given as follows:
2D convolution with $3\times3$ kernel and $32$ features;
($q$-)activation;
batch normalization;
2D convolution with $3\times3$ kernel and $32$ features;
($q$-)activation;
$2\times2$ max-pooling;
batch normalization;
2D convolution with $3\times3$ kernel and $64$ features;
($q$-)activation;
batch normalization;
2D convolution with $3\times3$ kernel and $64$ features;
($q$-)activation;
$2\times2$ max-pooling;
flatten into 1D vector;
batch normalization;
dense layer of output size $512$;
($q$-)activation;
batch normalization;
(optional) dropout layer with drop probability 0.2;
dense layer of output size $10$;
soft-max activation.

\item The MNIST-MLP architecture is:
dense layer of output size $256$;
($q$-)activation;
batch normalization;
(optional) dropout layer with drop probability 0.2;
dense layer of output size $256$;
($q$-)activation;
batch normalization;
(optional) dropout layer with drop probability $0.2$;
dense layer of output size $10$;
soft-max activation.

\item The CIFAR-CNN architecture is:
2D convolution with $3\times3$ kernel and $32$ features;
($q$-)activation;
2D convolution with $3\times3$ kernel and $32$ features;
($q$-)activation;
$2\times2$ max-pooling;
(optional) dropout layer with drop probability 0.2;
2D convolution with $3\times3$ kernel and $64$ features;
($q$-)activation;
2D convolution with $3\times3$ kernel and $64$ features;
($q$-)activation;
$2\times2$ max-pooling;
(optional) dropout layer with drop probability 0.2;
flatten into 1D vector;
dense layer of output size $512$;
($q$-)activation;
(optional) dropout layer with drop probability 0.1;
dense layer of output size $10$;
soft-max activation.
\end{itemize}

We use the cross-entropy as the loss function.
The model is trained for $100$ epochs
based on a stochastic gradient descent optimizer
with a mini-batch size of $64$ (MNIST) or $32$ (CIFAR) and
a learning rate of 0.05 (MNIST) or 0.01 (CIFAR) without momentum.
The learning rate is multiplied by $(1-10^{-6})$ after each mini-batch update.
We compare $\tanh$, $\relu$, $\elu$, $\softplus$ activations with their $q$-counterparts.
We either fix $\lambda_0=0.02$ or $0.1$, or anneal from $\lambda_0\in\{1,5,9\}$ with $\gamma=0.5$.
The learning curves are shown in \cref{fig:mnistcnn,fig:mnistmlp,fig:cifarcnn},
where the training curves show the sample-average cross-entropy values evaluated
on the training set after each epoch,
and the testing curves are classification accuracy.
In all figures, each training or testing curve is an average over $10$ independent runs.

For $q$-activation, $c$ means the $\lambda$ parameter is fixed ('c'onstant); $a$ means the 
$\lambda$ is 'a'nnealed based on \cref{eq:anneal}. For example, ``$c$0.02'' means $\lambda=0.02$ throughout the training process,
while ``$a$1'' means that $\lambda$ is annealed from $\lambda_0=1$.

We see that in almost all cases, $q$-activation can consistently improve
learning, in the sense that both training and testing errors are reduced.
This implies that $q$-neurons can get to a better local optimum as compared to
the corresponding deterministic neurons.
The exception worth noting is $q$-$\relu$, which cannot improve over $\relu$ activation.
This is because $g_q(x)$ is very similar to the original $f(x)$ for (piece-wisely) linear functions.
By \cref{thm:limit}, $f''(x)=0$ implies that the gradient of 
$g_q(x)$ and $f(x)$ are similar for small $\Var(q)$.
One is advised to use $q$-neurons only with {\em curved activation
functions} such as $\elu$, $\tanh$, etc.

We also observe that the benefits of $q$-neurons are not sensitive to
hyper-parameter selection.  In almost all cases, $q$-neuron with
$\lambda$ simply fixed to 0.02/0.1 can bring better generalization
performance, while an annealing scheme can further improve the score.
Setting $\lambda$ too large may lead to under-fit.
One can benefit from $q$-neurons either with or without dropout.

On the MNIST dataset, the best performance with error rate 0.35\% (99.65\% accuracy)
is achieved by the CNN architecture with $q$-$\elu$ and $q$-$\tanh$.
On the CIFAR10 dataset, the best performance of the CNN with accuracy 82.9\%
is achieved by $q$-$\elu$.

\begin{figure*}
\centering
\begin{subfigure}[b]{\myfigwidth}
\includegraphics[width=\textwidth]{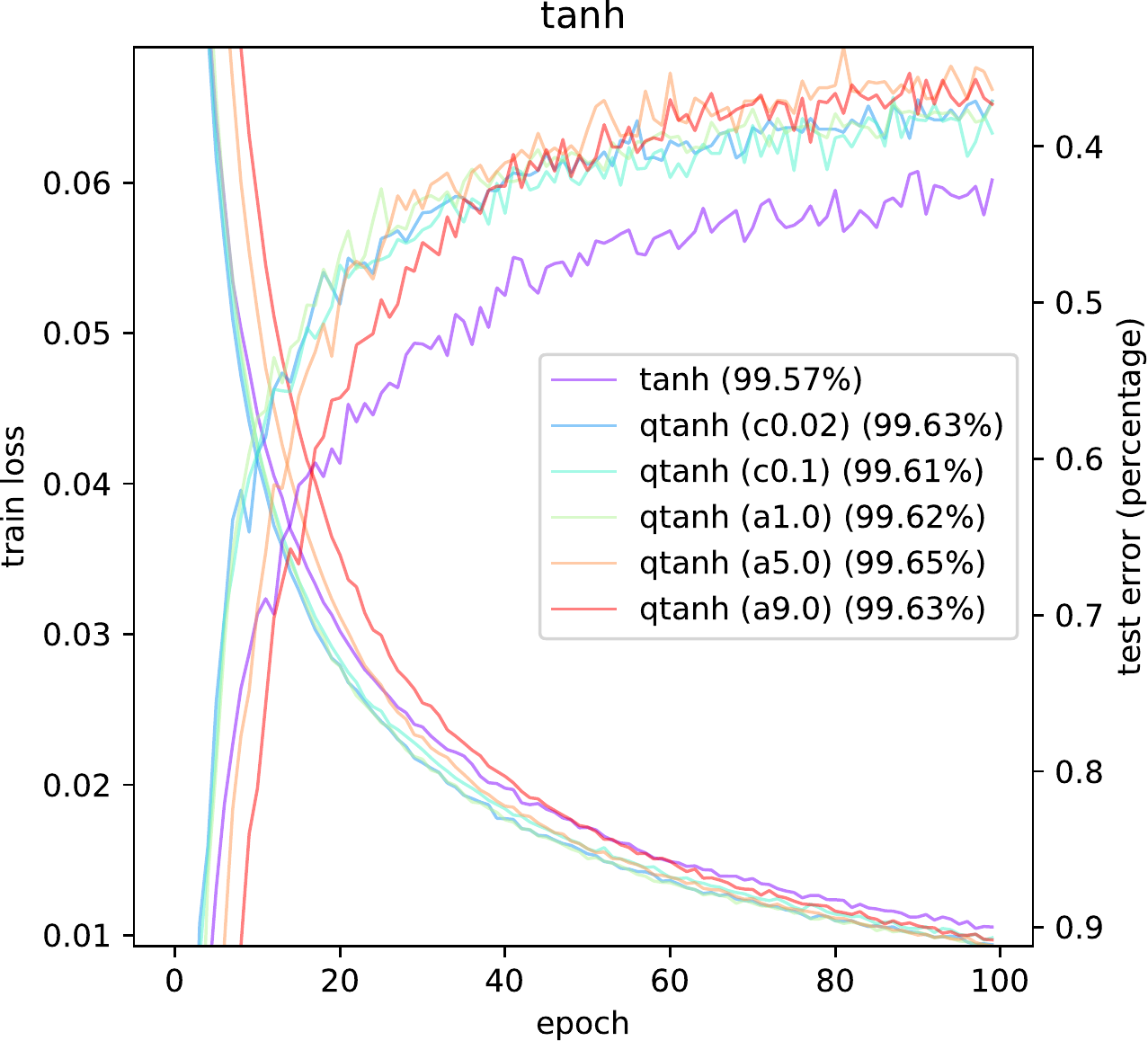}
\end{subfigure}
\begin{subfigure}[b]{\myfigwidth}
\includegraphics[width=\textwidth]{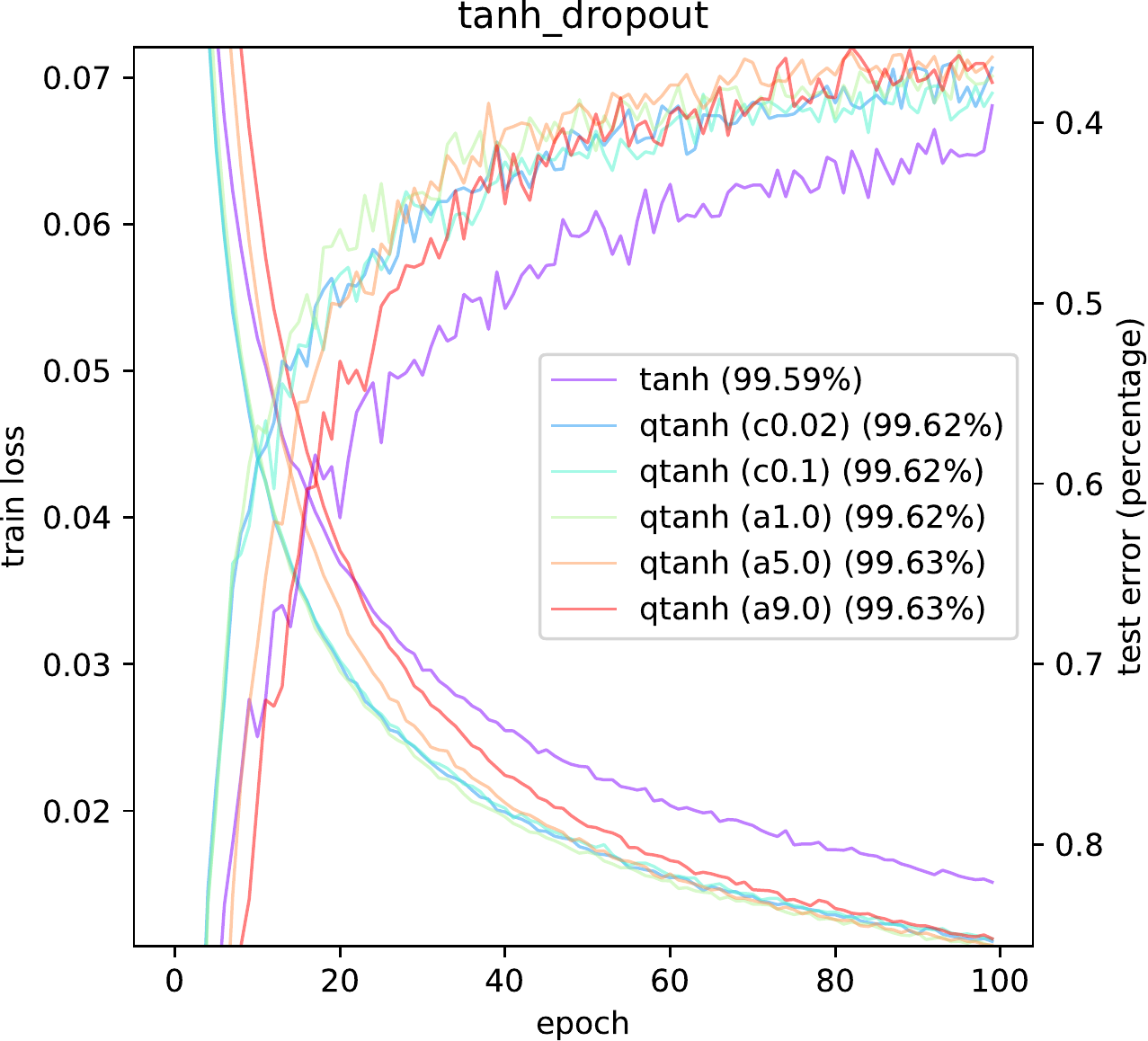}
\end{subfigure}
\begin{subfigure}[b]{\myfigwidth}
\includegraphics[width=\textwidth]{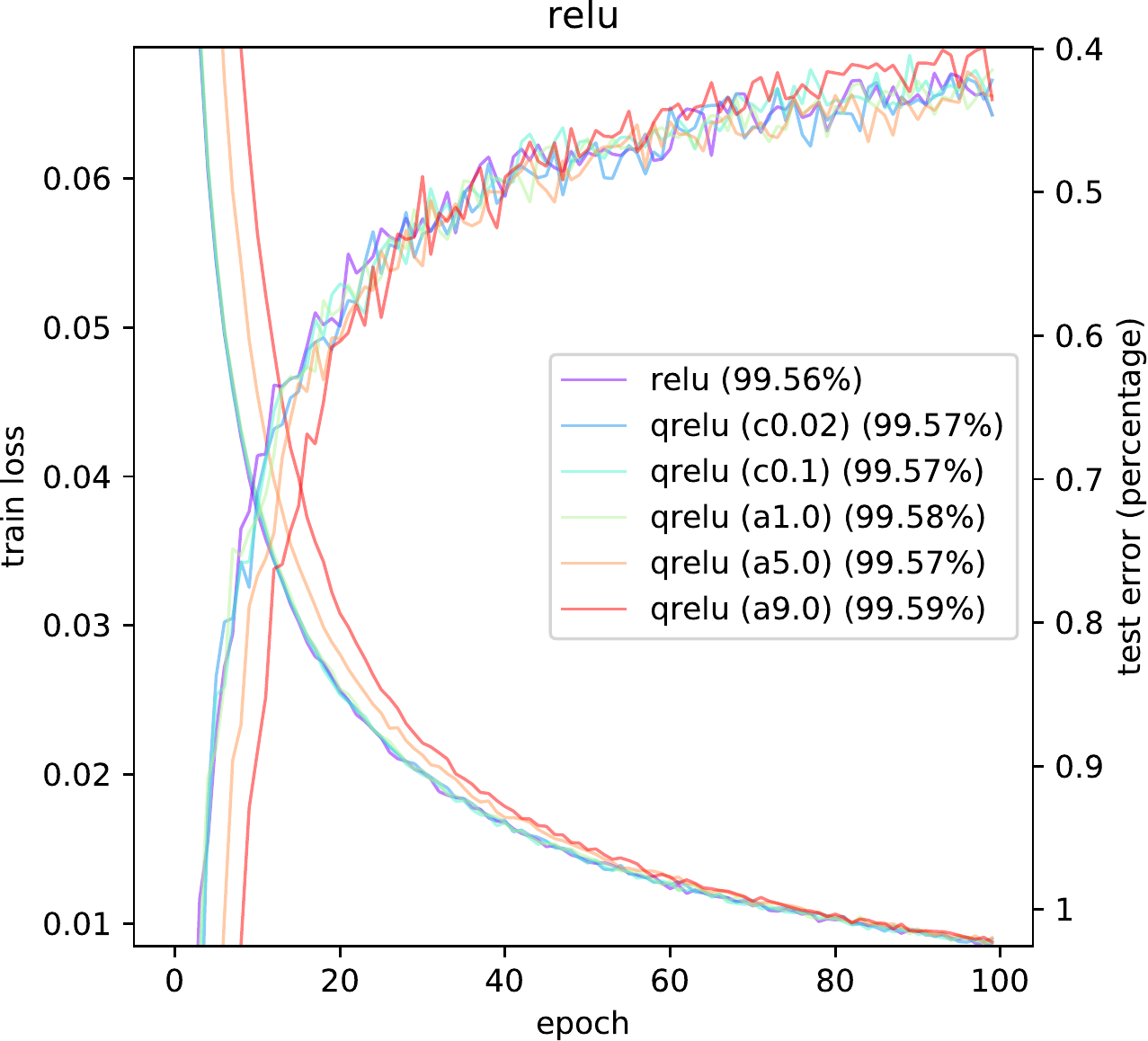}
\end{subfigure}
\begin{subfigure}[b]{\myfigwidth}
\includegraphics[width=\textwidth]{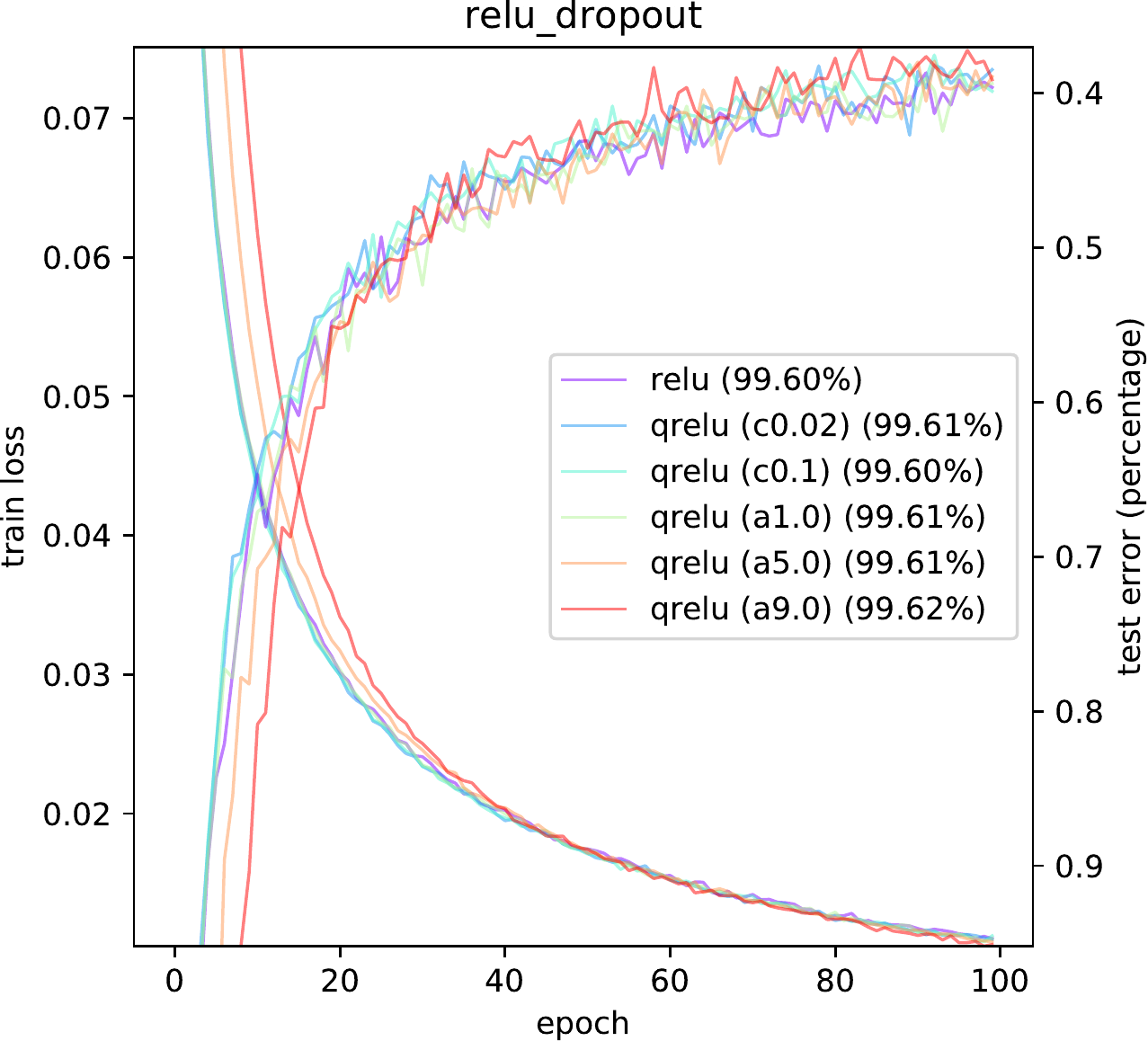}
\end{subfigure}
\begin{subfigure}[b]{\myfigwidth}
\includegraphics[width=\textwidth]{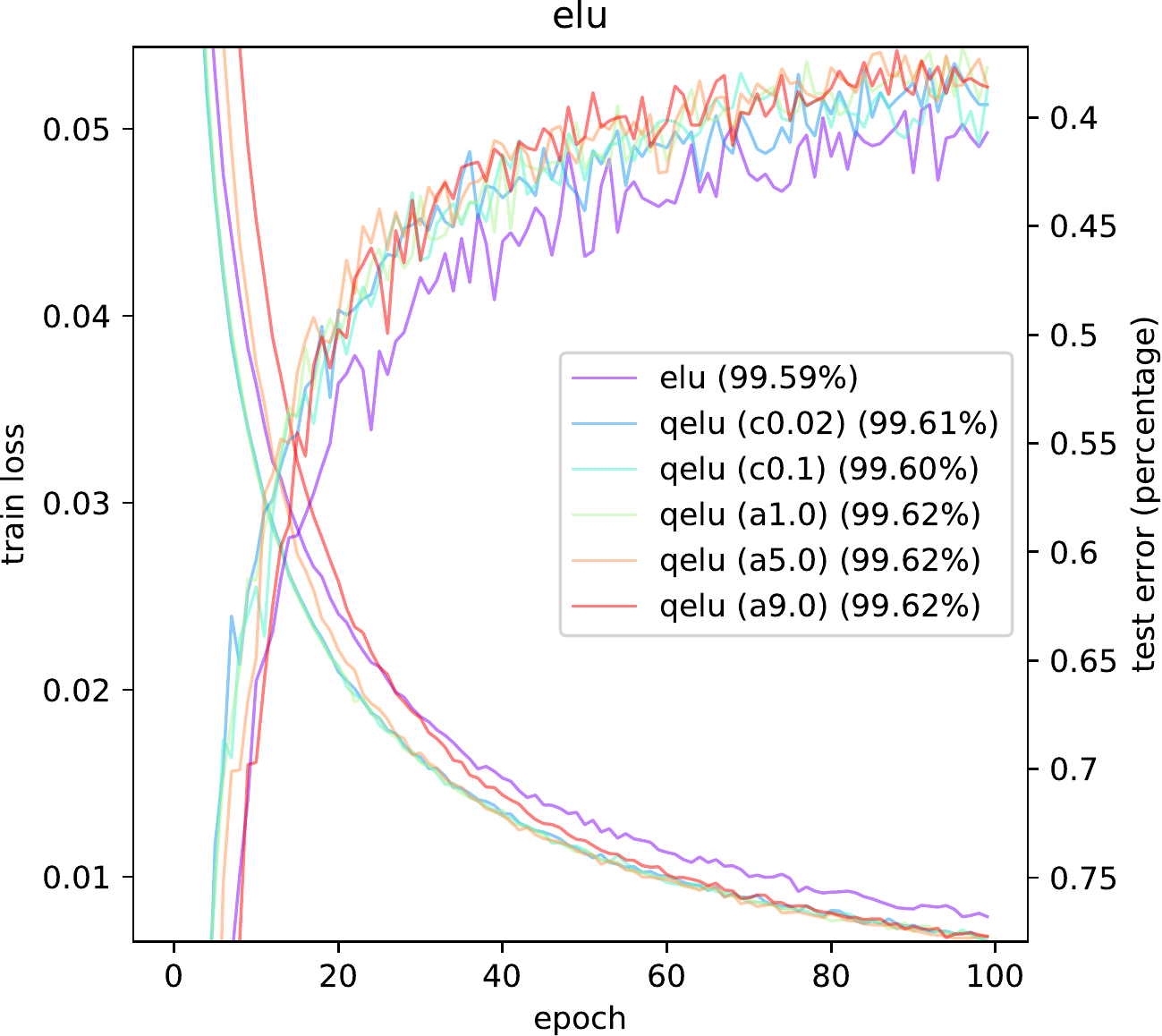}
\end{subfigure}
\begin{subfigure}[b]{\myfigwidth}
\includegraphics[width=\textwidth]{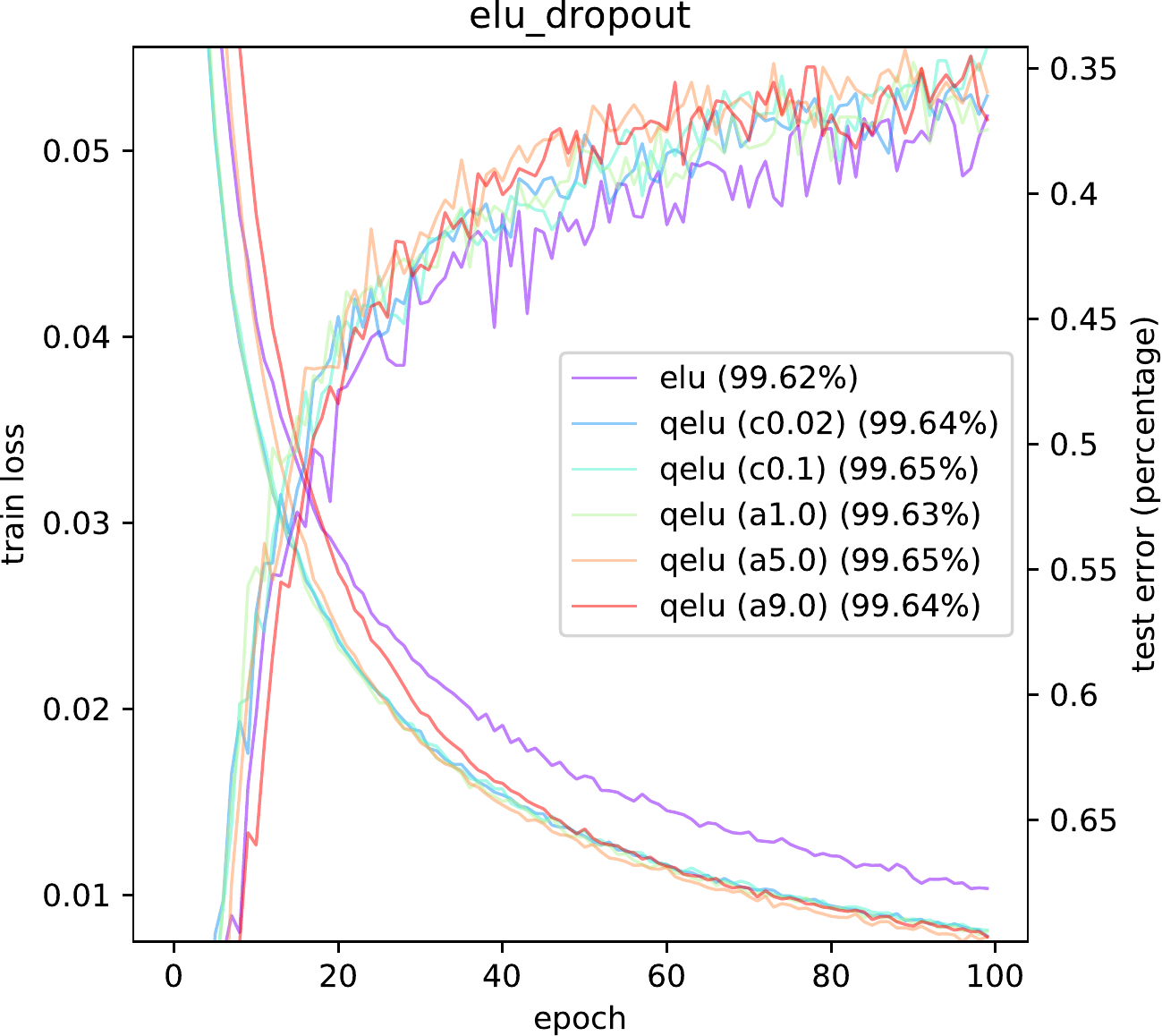}
\end{subfigure}
\begin{subfigure}[b]{\myfigwidth}
\includegraphics[width=\textwidth]{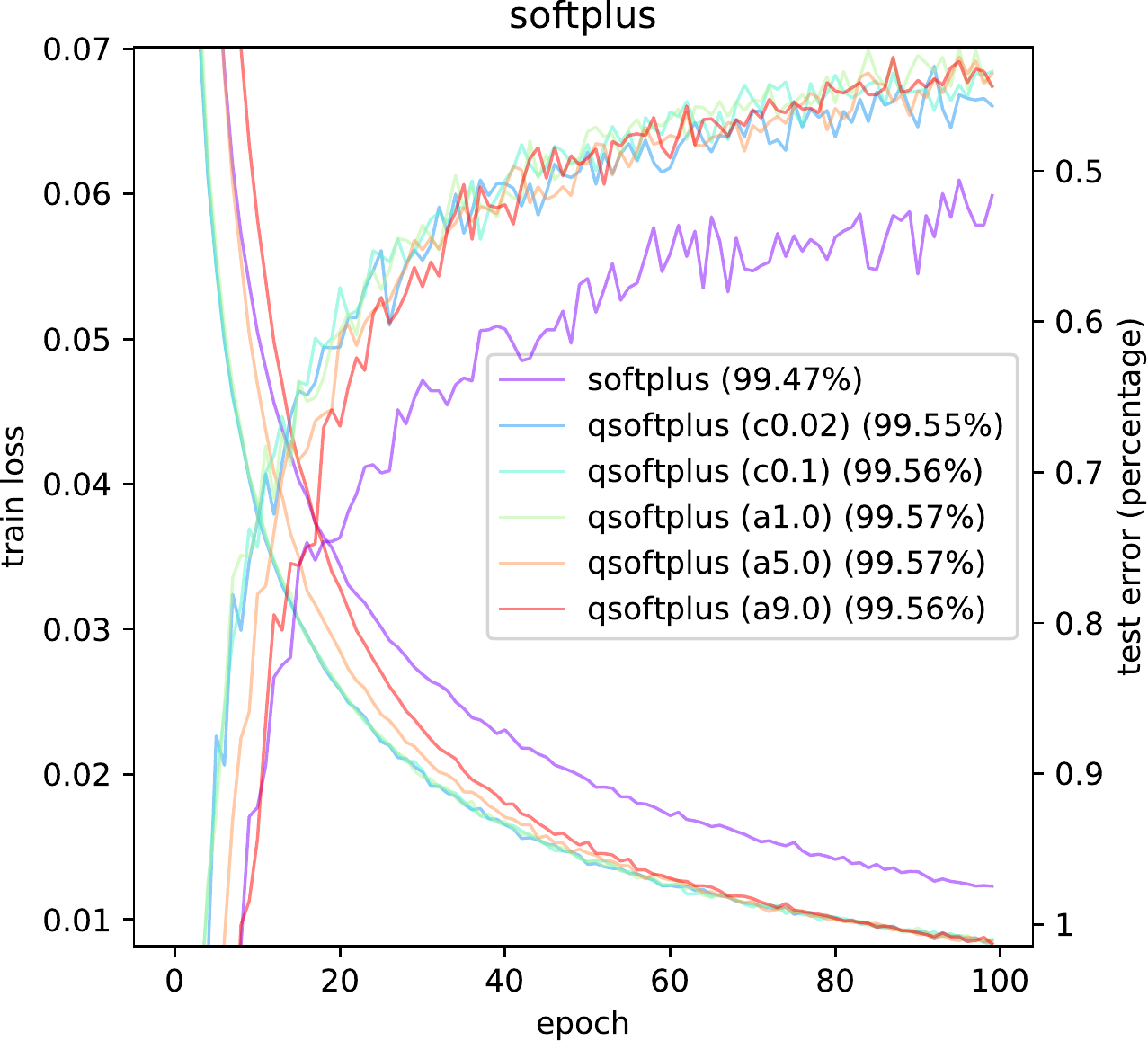}
\end{subfigure}
\begin{subfigure}[b]{\myfigwidth}
\includegraphics[width=\textwidth]{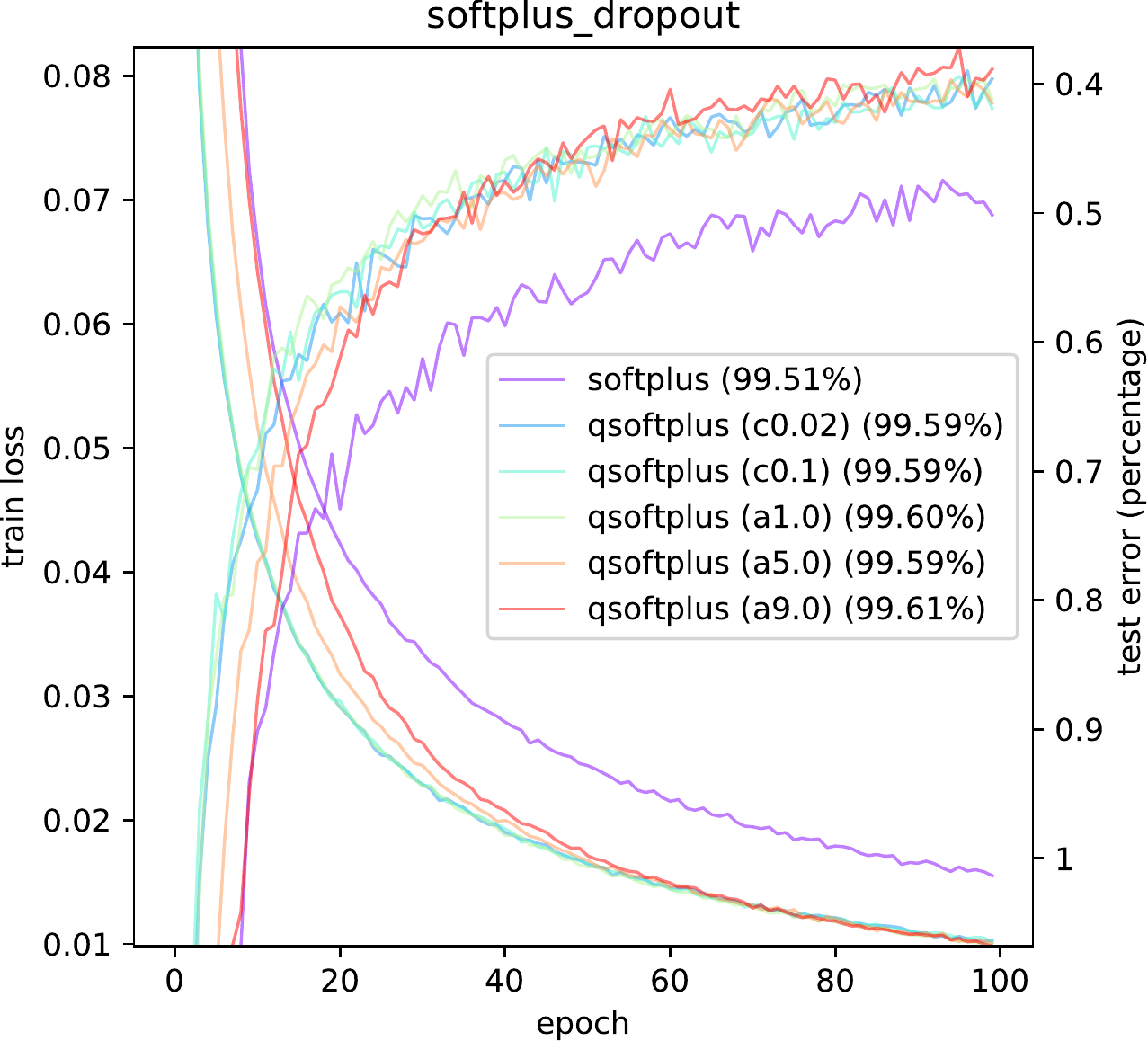}
\end{subfigure}
\caption{Training loss (descending curves) and testing accuracy (ascending curves) of a CNN on the MNIST dataset,
using different activation functions (from top to bottom), with (left) or without (right) dropout.}\label{fig:mnistcnn}
\end{figure*}

\begin{figure*}
\centering
\begin{subfigure}[b]{\myfigwidth}
\includegraphics[width=\textwidth]{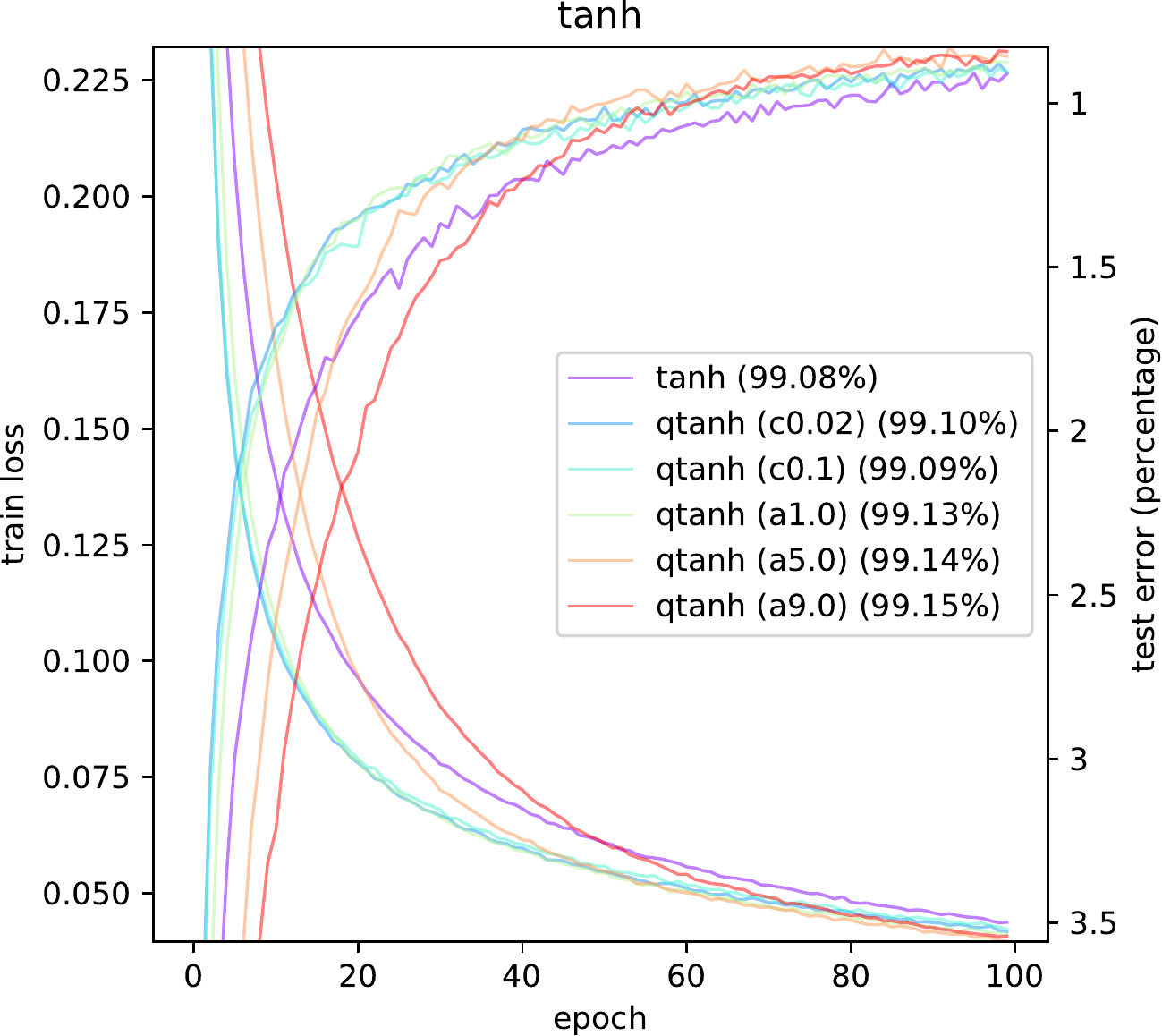}
\end{subfigure}
\begin{subfigure}[b]{\myfigwidth}
\includegraphics[width=\textwidth]{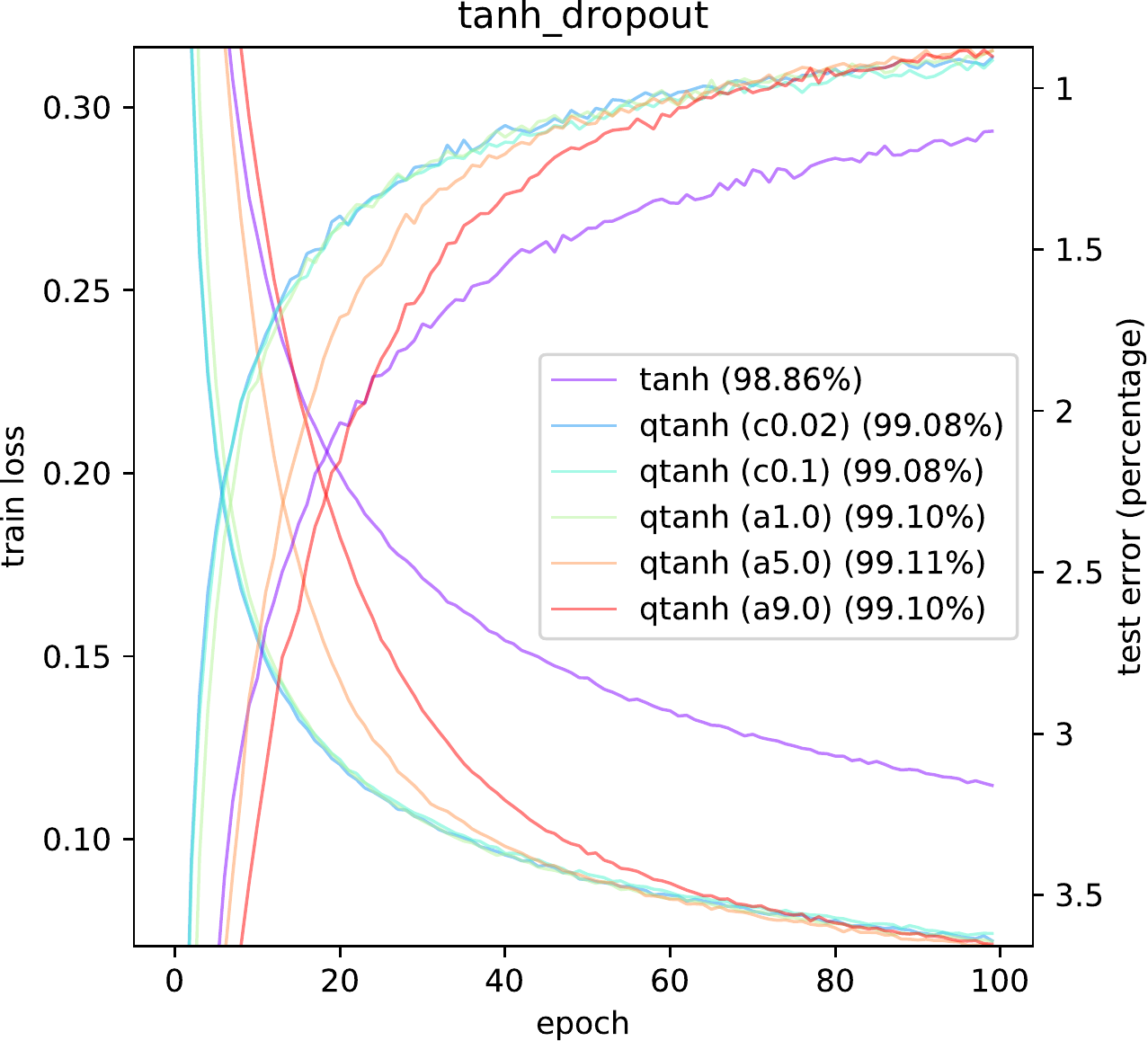}
\end{subfigure}
\begin{subfigure}[b]{\myfigwidth}
\includegraphics[width=\textwidth]{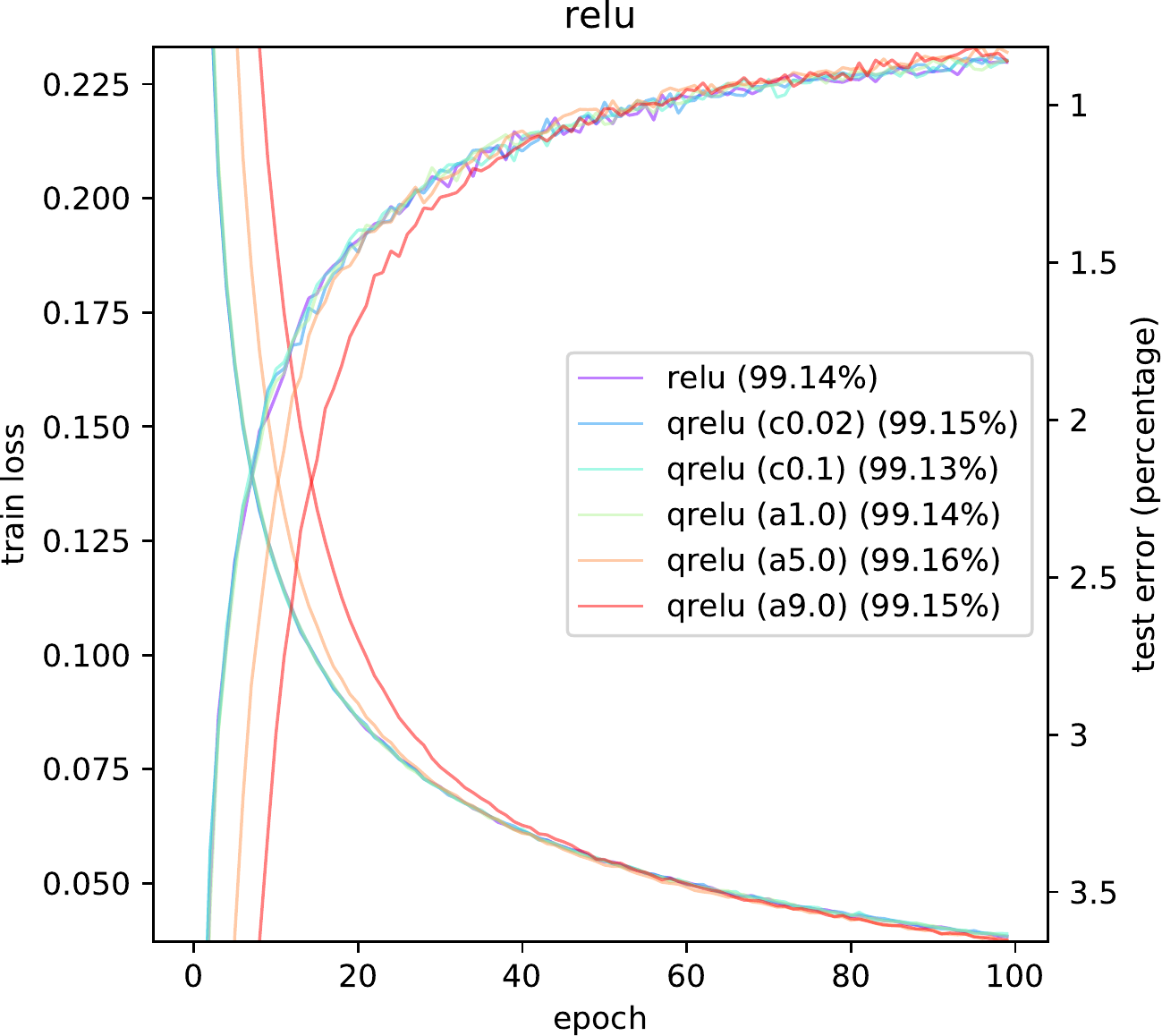}
\end{subfigure}
\begin{subfigure}[b]{\myfigwidth}
\includegraphics[width=\textwidth]{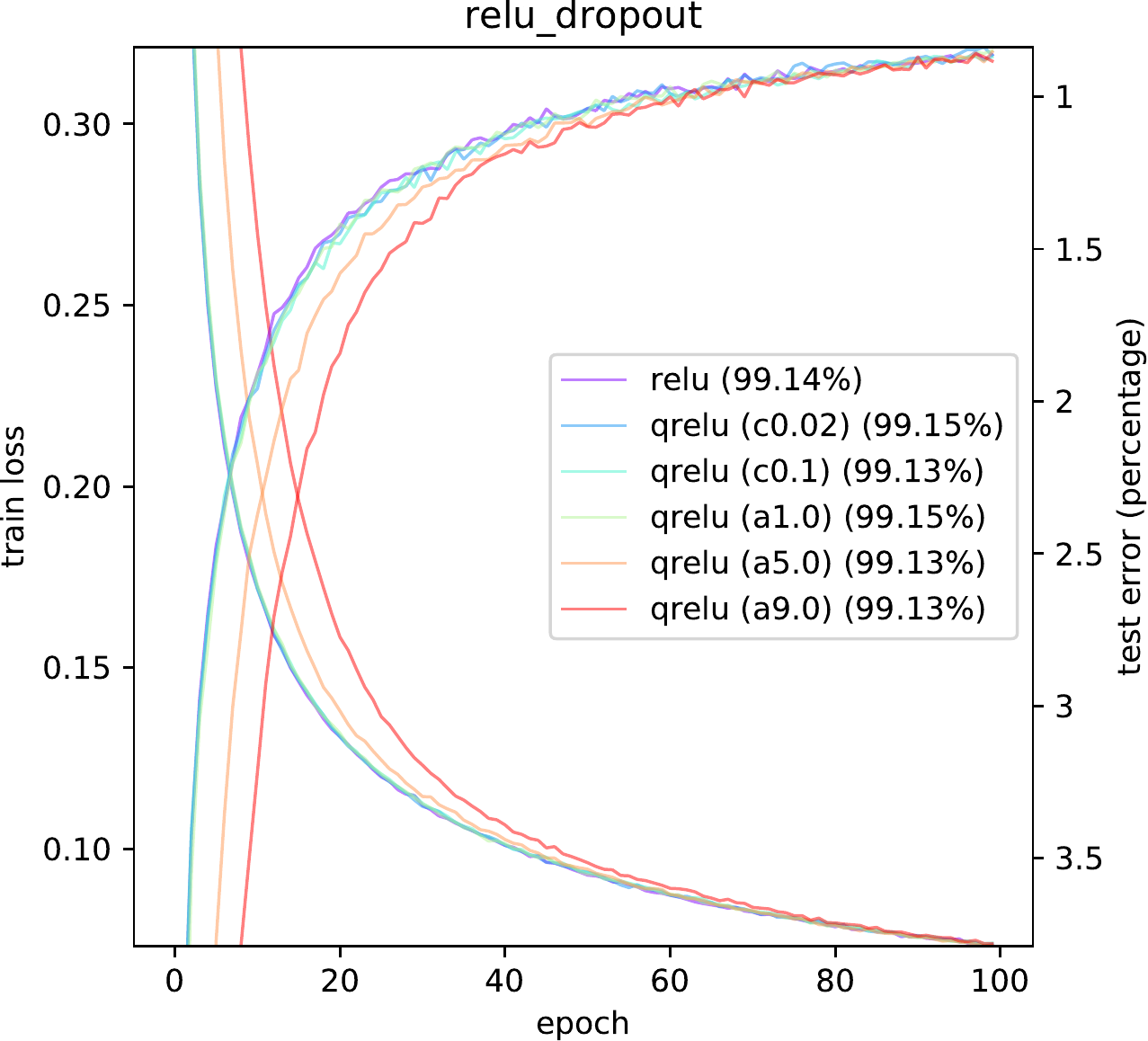}
\end{subfigure}
\begin{subfigure}[b]{\myfigwidth}
\includegraphics[width=\textwidth]{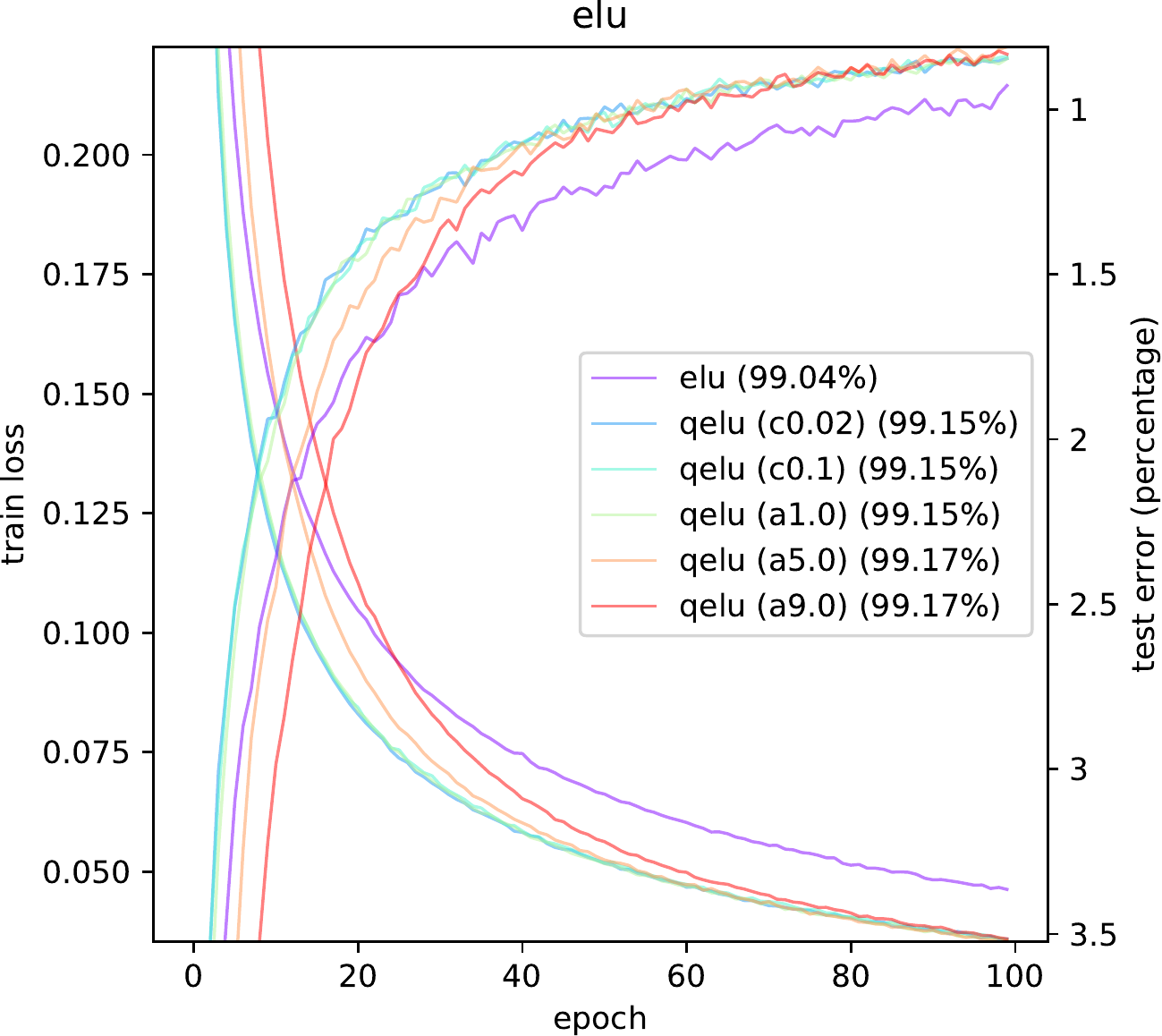}
\end{subfigure}
\begin{subfigure}[b]{\myfigwidth}
\includegraphics[width=\textwidth]{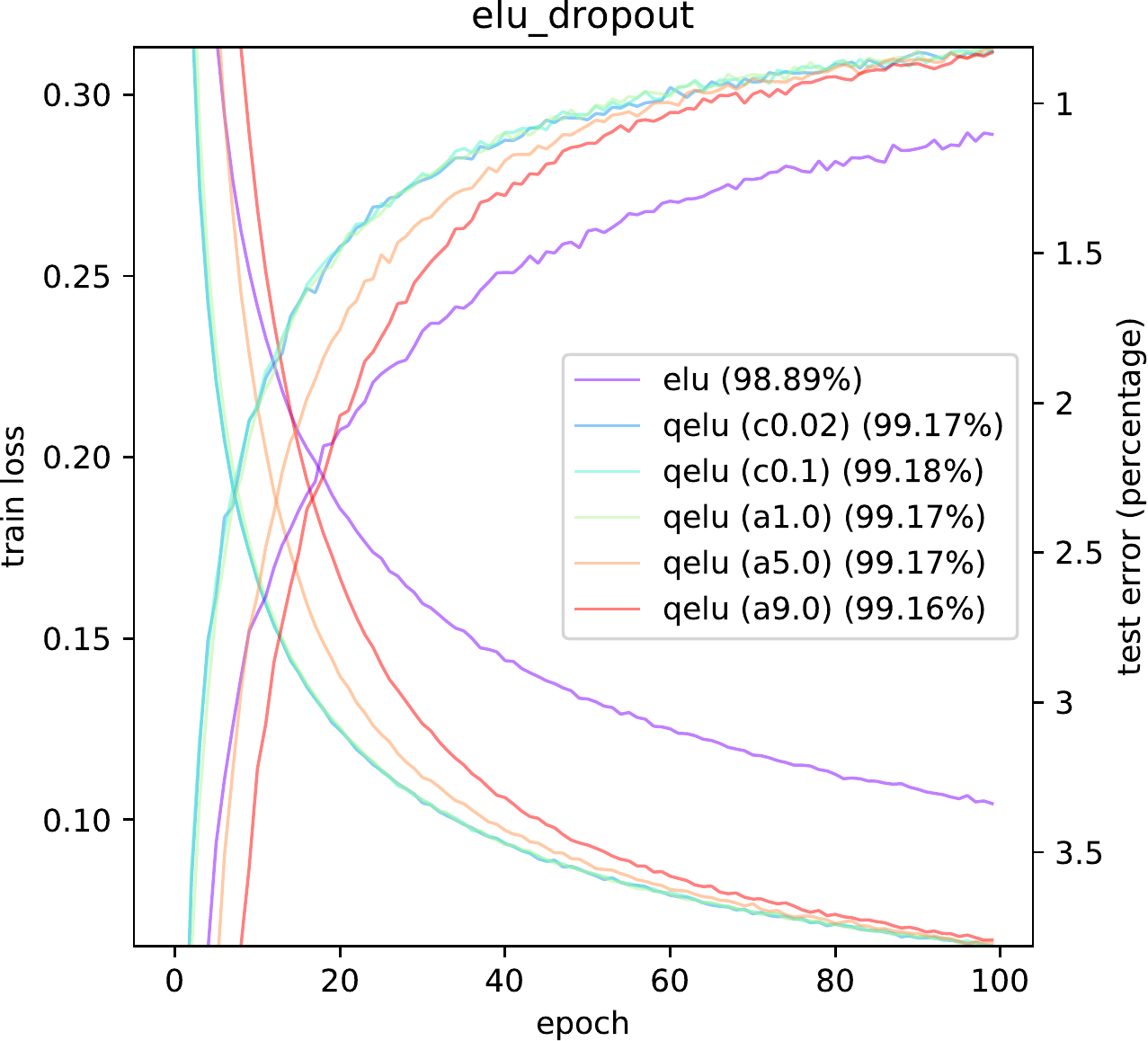}
\end{subfigure}
\begin{subfigure}[b]{\myfigwidth}
\includegraphics[width=\textwidth]{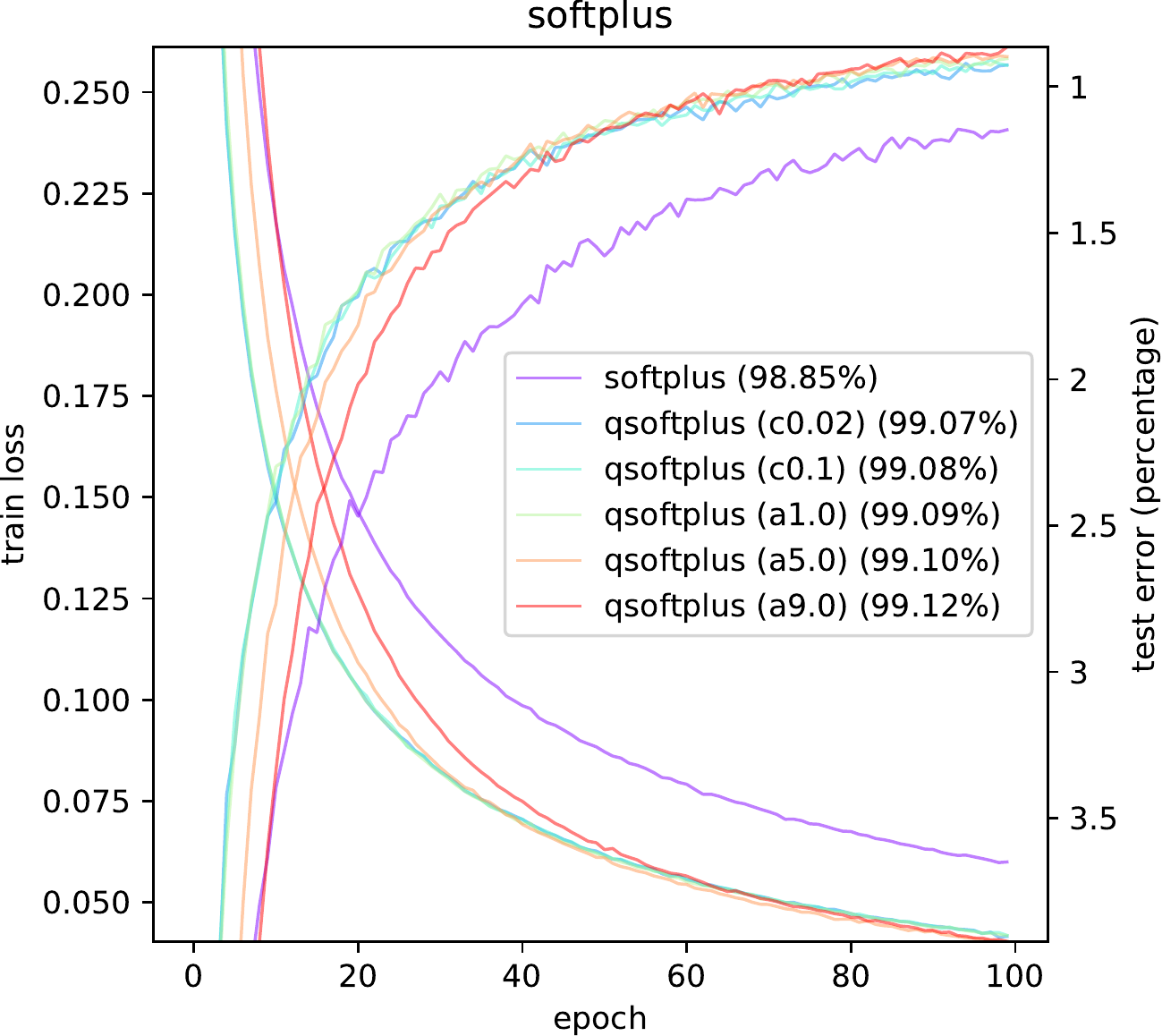}
\end{subfigure}
\begin{subfigure}[b]{\myfigwidth}
\includegraphics[width=\textwidth]{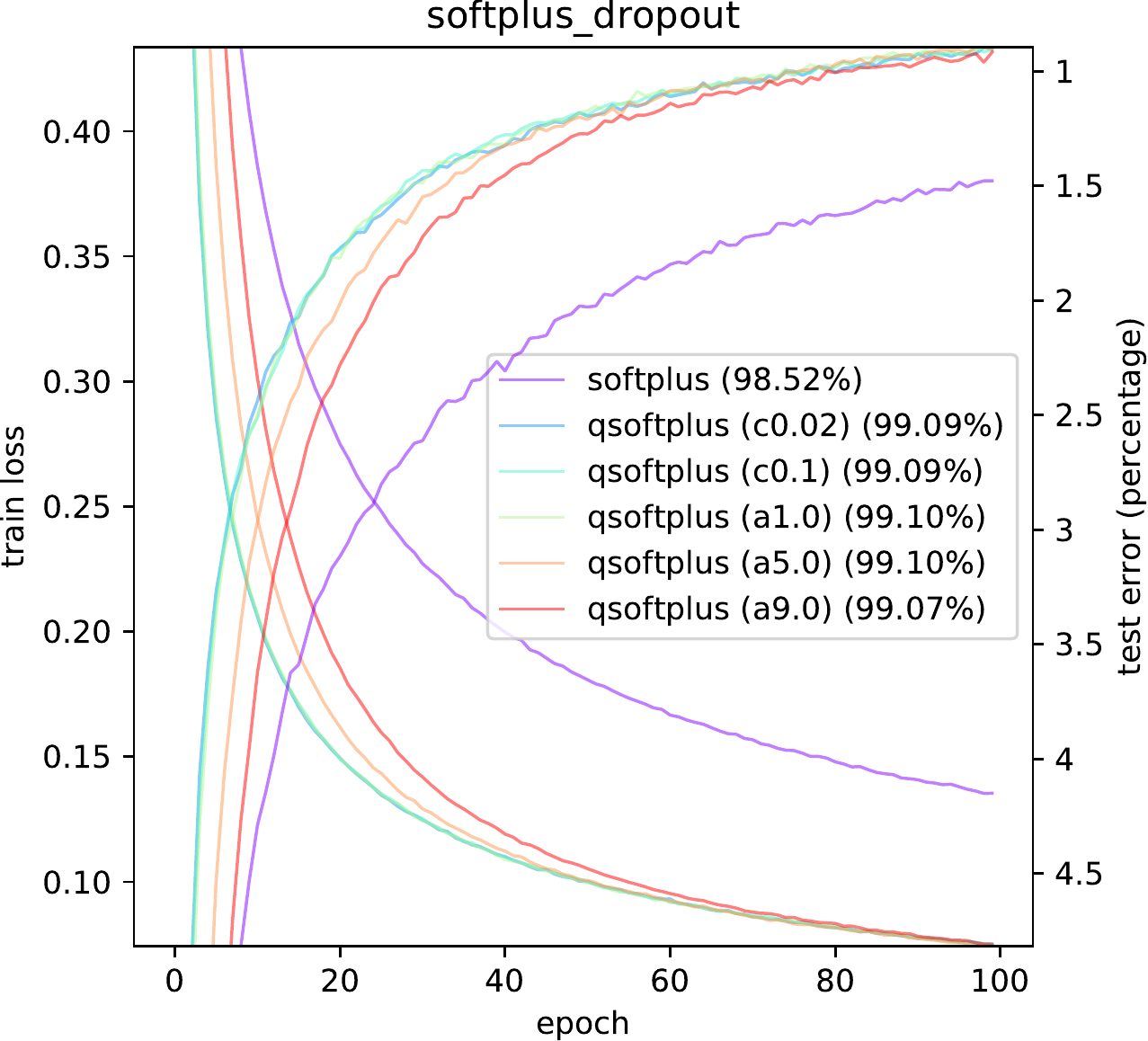}
\end{subfigure}
\caption{Training loss (descending curves) and testing accuracy (ascending curves) of a MLP on MNIST.}\label{fig:mnistmlp}
\end{figure*}

\begin{figure*}
\centering
\begin{subfigure}[b]{\myfigwidth}
\includegraphics[width=\textwidth]{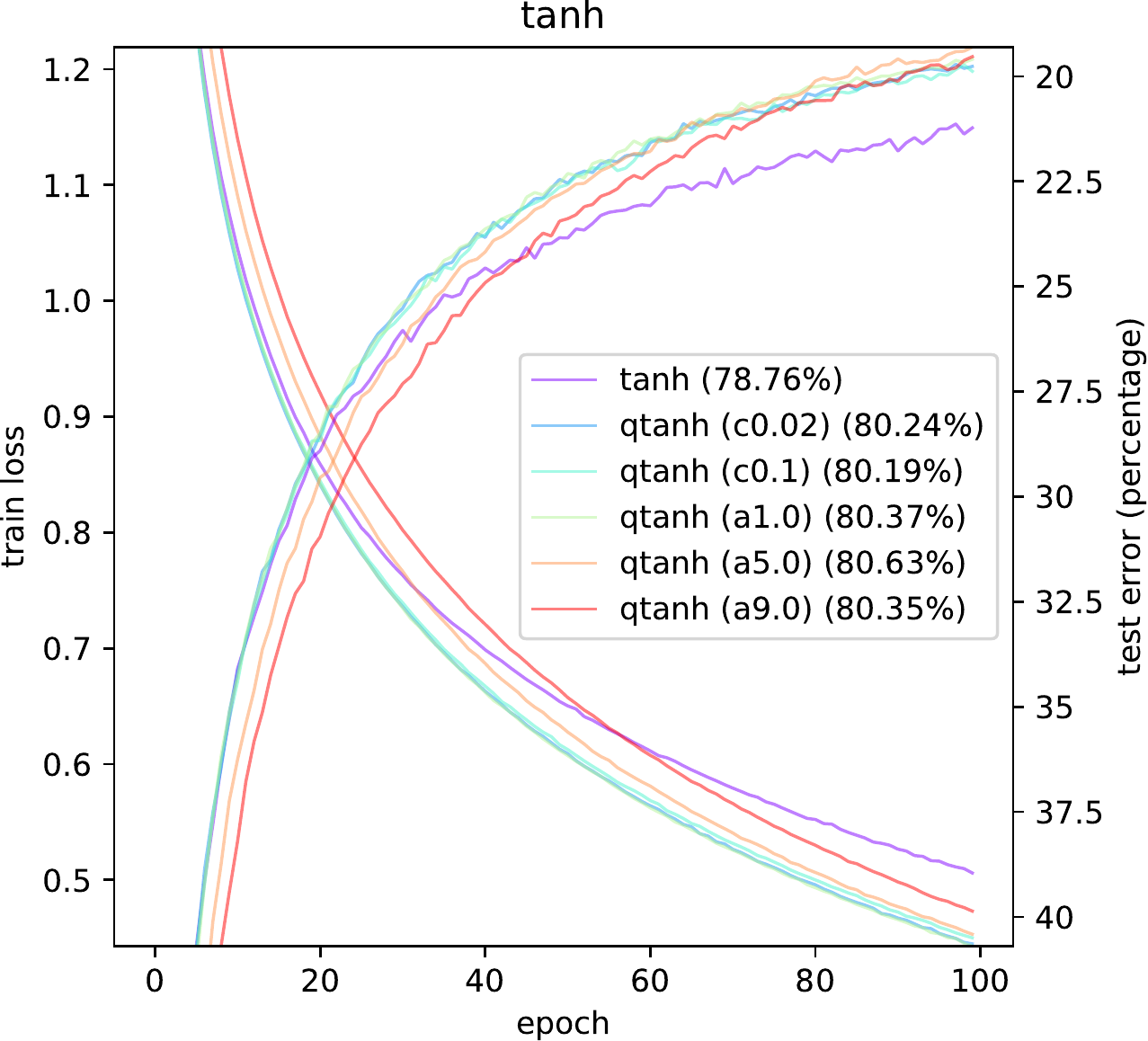}
\end{subfigure}
\begin{subfigure}[b]{\myfigwidth}
\includegraphics[width=\textwidth]{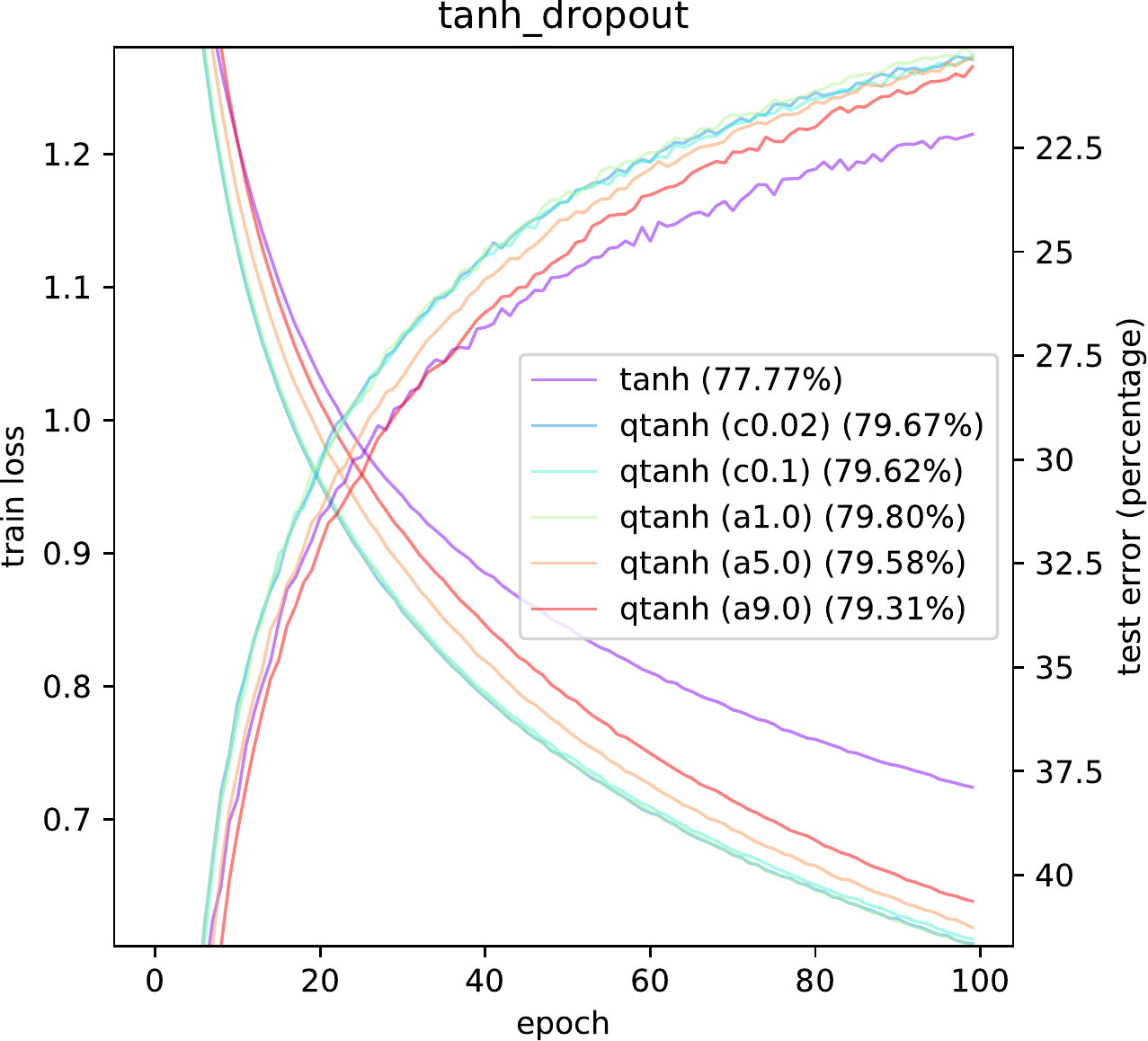}
\end{subfigure}
\begin{subfigure}[b]{\myfigwidth}
\includegraphics[width=\textwidth]{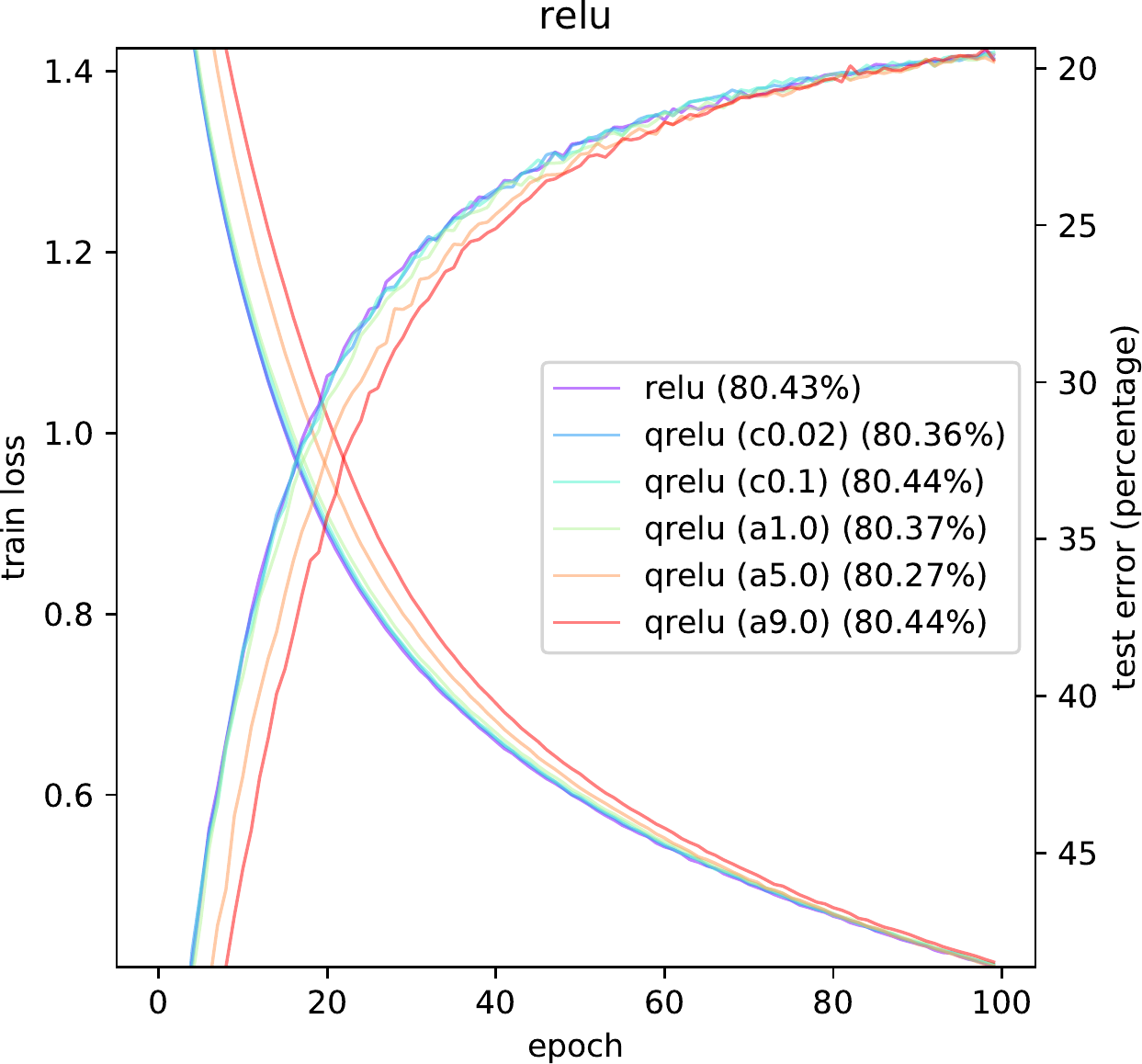}
\end{subfigure}
\begin{subfigure}[b]{\myfigwidth}
\includegraphics[width=\textwidth]{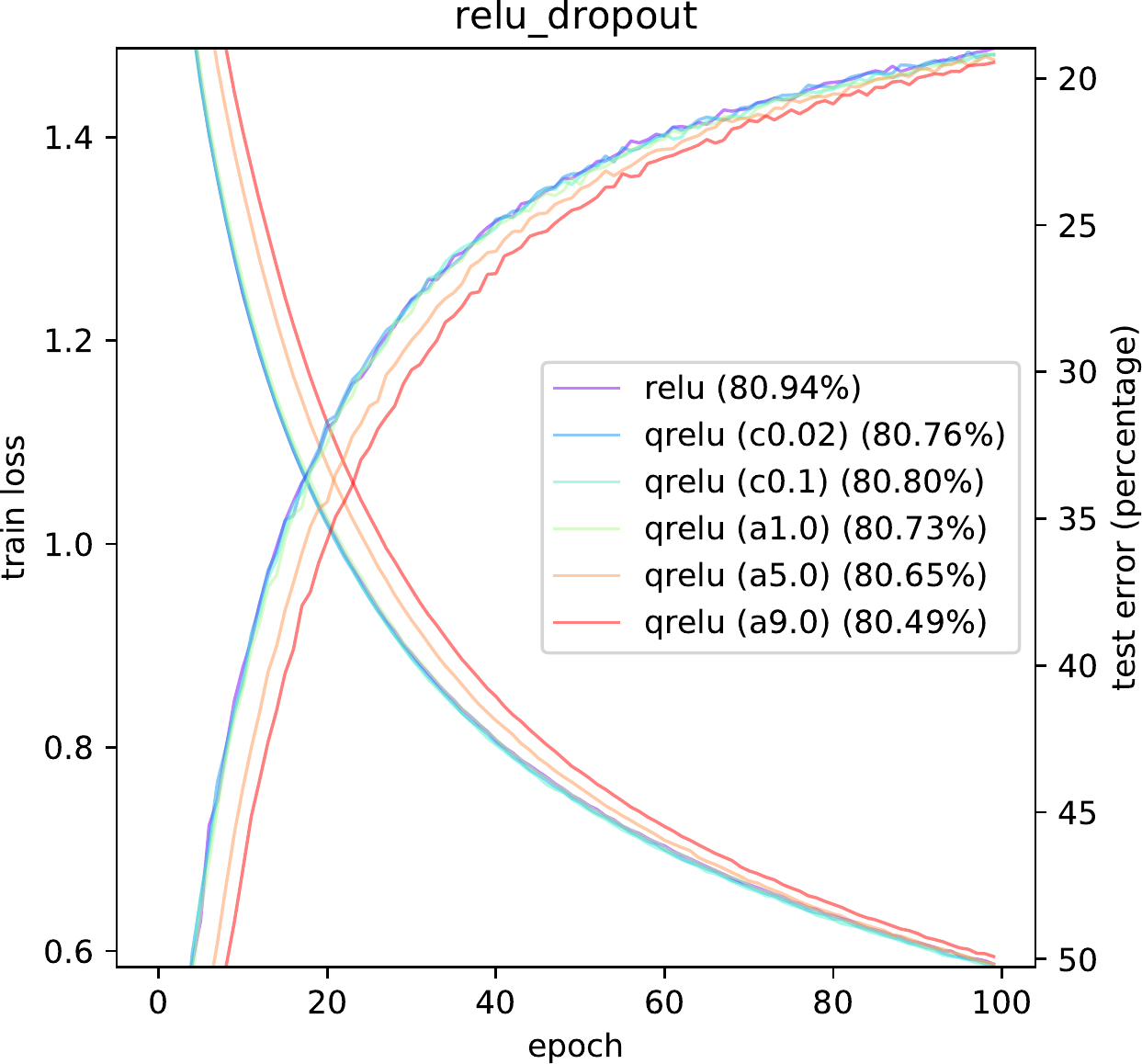}
\end{subfigure}
\begin{subfigure}[b]{\myfigwidth}
\includegraphics[width=\textwidth]{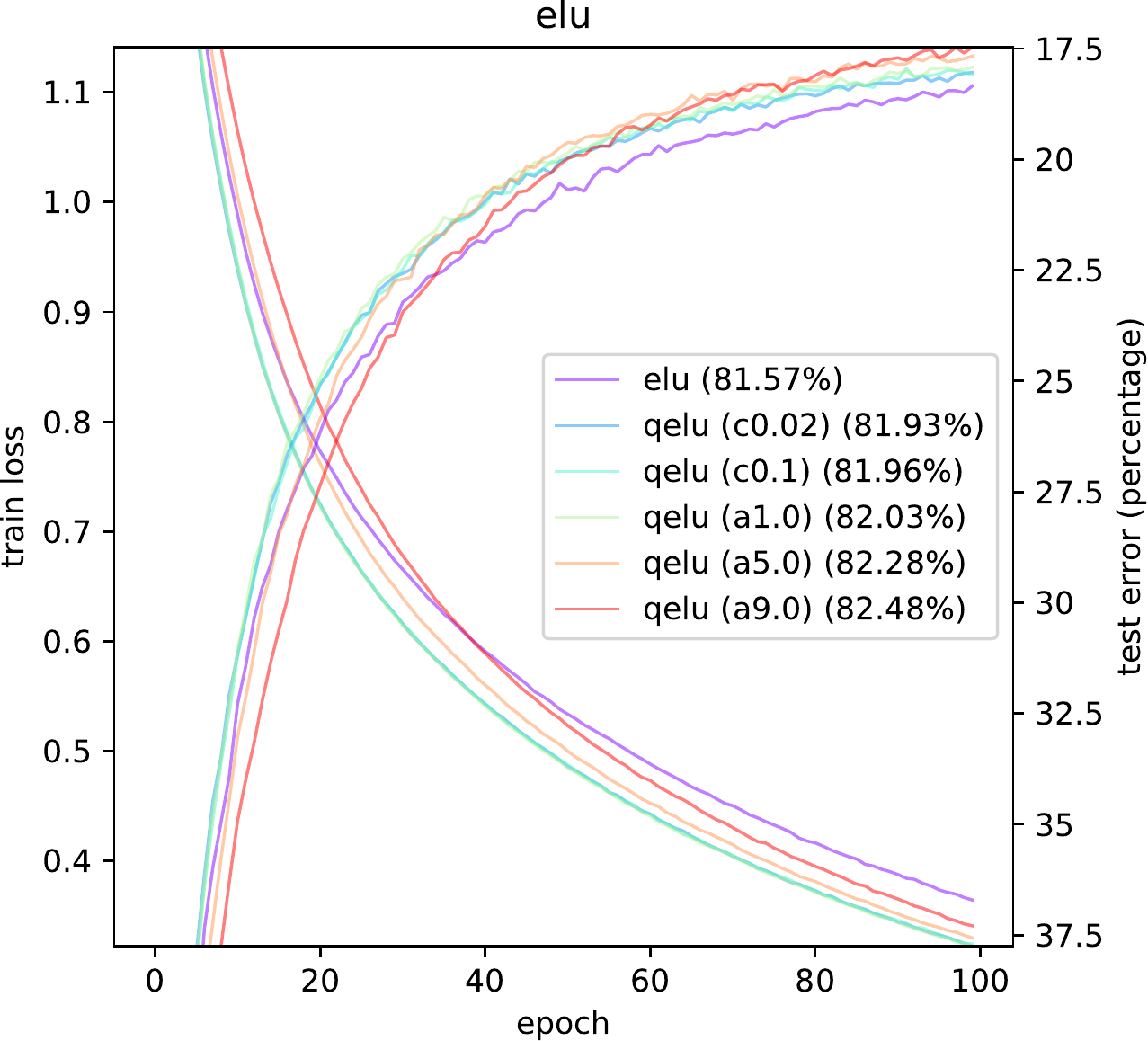}
\end{subfigure}
\begin{subfigure}[b]{\myfigwidth}
\includegraphics[width=\textwidth]{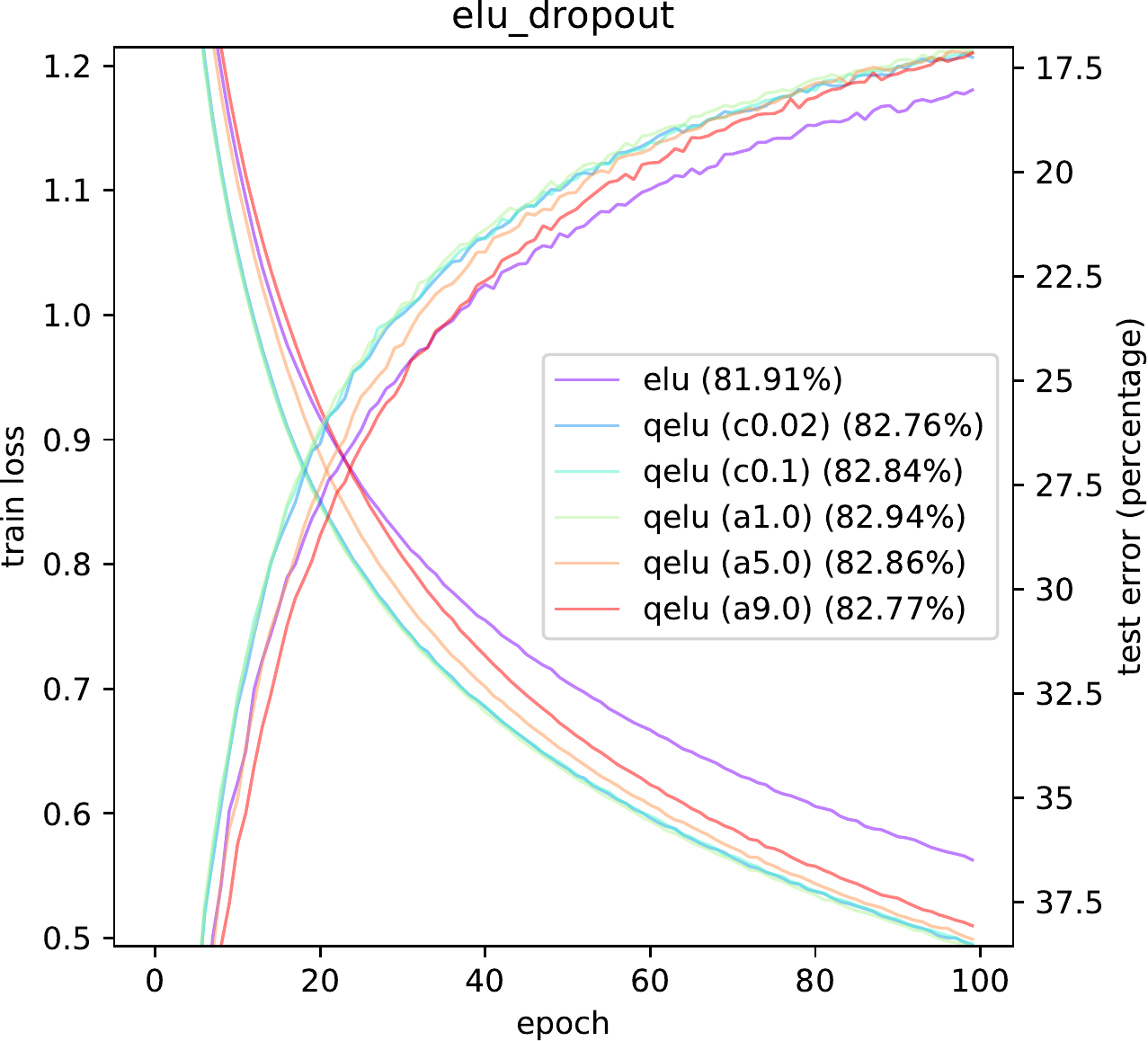}
\end{subfigure}
\begin{subfigure}[b]{\myfigwidth}
\includegraphics[width=\textwidth]{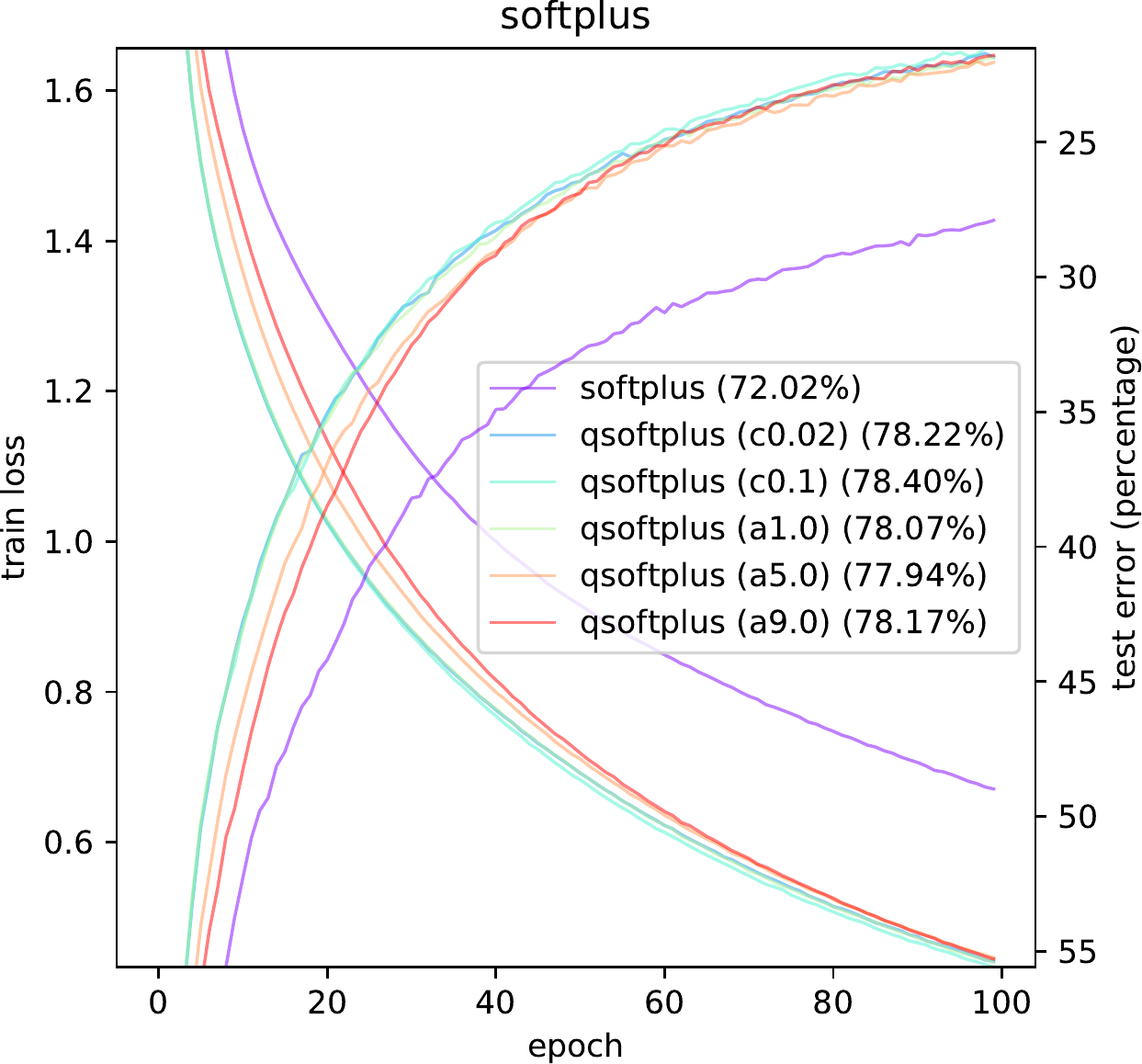}
\end{subfigure}
\begin{subfigure}[b]{\myfigwidth}
\includegraphics[width=\textwidth]{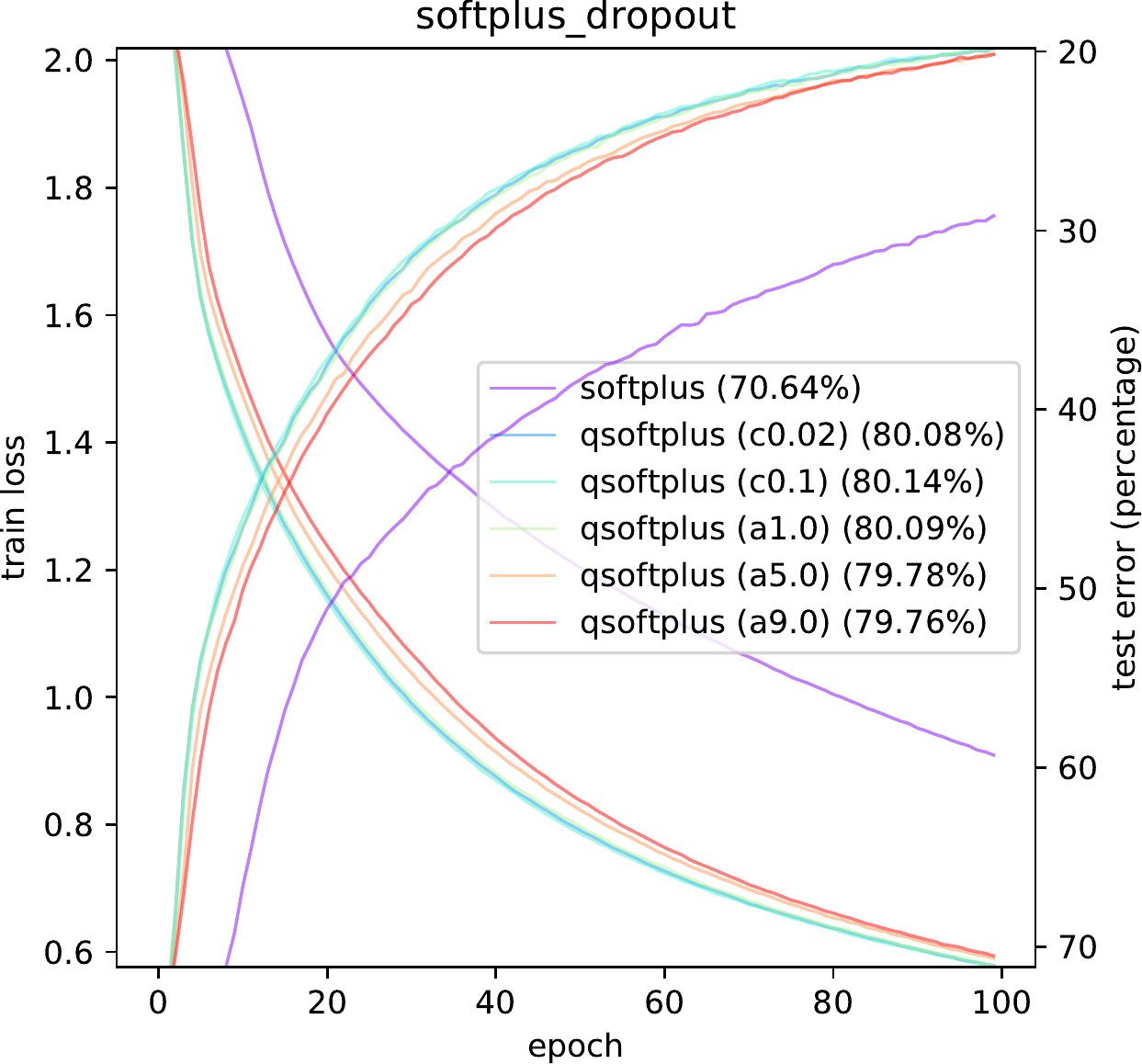}
\end{subfigure}
\caption{Training loss (descending curves) and testing accuracy (ascending curves) of a CNN on CIFAR10.}\label{fig:cifarcnn}
\end{figure*}

\section{Conclusion}
We proposed the stochastic $q$-neurons based on converting activation functions into corresponding stochastic $q$-activation functions using Jackson's $q$-calculus.
We found experimentally that $q$-neurons can consistently (although slightly) improve the
generalization performance, and can goes deeper in the error surface.


\begin{thebibliography}{10}

\bibitem{Jackson-1909}
Frank~Hilton Jackson,
\newblock ``On $q$-functions and a certain difference operator,''
\newblock {\em Earth and Environmental Science Transactions of The Royal
  Society of Edinburgh}, vol. 46, no. 2, pp. 253--281, 1909.

\bibitem{qbook-2001}
Victor Kac and Pokman Cheung,
\newblock {\em Quantum Calculus},
\newblock Universitext. Springer New York, 2001.

\bibitem{khan18}
Shujaat Khan, Alishba Sadiq, Imran Naseem, Roberto Togneri, and Mohammed
  Bennamoun,
\newblock ``Enhanced $q$-least mean square,''
\newblock 2018,
\newblock arXiv:1801.00410 [math.OC].

\bibitem{qGradOpt-2016}
{\'E}rica~J. C. Gouv{\^e}a, Rommel~G. Regis, Aline~C. Soterroni, Marluce~C.
  Scarabello, and Fernando~M. Ramos,
\newblock ``Global optimization using $q$-gradients,''
\newblock {\em European Journal of Operational Research}, vol. 251, no. 3, pp.
  727--738, 2016.

\bibitem{maas13}
Andrew~L. Maas, Awni~Y. Hannun, and Andrew~Y. Ng,
\newblock ``Rectifier non-linearities improve neural network acoustic models,''
\newblock in {\em ICML}, 2013.

\bibitem{clevert15}
Djork{-}Arn{\'{e}} Clevert, Thomas Unterthiner, and Sepp Hochreiter,
\newblock ``Fast and accurate deep network learning by exponential linear units
  ({ELU}s),''
\newblock in {\em ICLR}, 2016,
\newblock arXiv 1511.07289.

\bibitem{qbook-2012}
Thomas Ernst,
\newblock {\em A comprehensive treatment of $q$-calculus},
\newblock Springer Science \& Business Media, 2012.

\bibitem{pq-2013}
Patrick Njionou Sadjang,
\newblock ``On the fundamental theorem of $(p,q)$-calculus and some $(p,
  q)$-{T}aylor formulas,''
\newblock {\em arXiv preprint arXiv:1309.3934}, 2013.

\bibitem{BD-2005}
Arindam Banerjee, Srujana Merugu, Inderjit~S Dhillon, and Joydeep Ghosh,
\newblock ``Clustering with {B}regman divergences,''
\newblock {\em Journal of machine learning research}, vol. 6, no. Oct, pp.
  1705--1749, 2005.

\bibitem{qgrad-2013}
Aline~C. Soterroni, Roberto~L. Galski, and Fernando~M. Ramos,
\newblock ``The $q$-gradient method for continuous global optimization,''
\newblock in {\em AIP Conference Proceedings}. AIP, 2013, vol. 1558, pp.
  2389--2393.
	
\bibitem{xn-2018}
Zhen Xu and Frank Nielsen,
\newblock ``Beyond Ordinary Stochastic Gradient Descent,''
\newblock {\em preprint INF517}, Ecole Polytechnique, March 2018.

\bibitem{gradientnoise}
Arvind Neelakantan, Luke Vilnis, Quoc V. Le, Ilya Sutskever, Lukasz Kaiser, Karol Kurach and James Martens,
\newblock ``Adding Gradient Noise Improves Learning for Very Deep Networks,''
\newblock in {\em ICLR}, 2016,
\newblock arXiv 1511.06807.

\bibitem{dropout}
Nitish Srivastava, Geoffrey Hinton, Alex Krizhevsky, Ilya Sutskever and Ruslan Salakhutdinov,
\newblock ``Dropout: A Simple Way to Prevent Neural Networks from Overfitting,''
\newblock {\em Journal of machine learning research},
vol. 15, no. Jun, pp. 1929--1958, 2014.

\end{thebibliography}

\appendix

\section{Brief overview of the $(p,q)$-differential calculus}\label{sec:pqgradient}

For $p,q\in\bbR, p\not =q$, define the $(p,q)$-differential:
	$$
	d_{p,q} f(x) \eqdef f(qx)-f(px).
	$$
	In particular, $d_{p,q} x=(q-p)x$.
	
The $(p,q)$-derivative is then obtained as:
	$$
	D_{p,q} f(x) \eqdef \frac{d_{p,q} f(x)}{d_{p,q} x} =  \frac{f(qx)-f(px)}{(q-p)x}.
	$$
	We have $D_{p,q} f(x)=D_{q,p} f(x)$.

Consider a real-valued scalar function $f(x)$.
The differential operator $D$ consists in taking the derivative: $D f(x)=\frac{d}{\dx}=f'(x)$.

The {\em $(p,q)$-differential operator} $D_{p,q}$  for two distinct scalars $p$ and $q$ is defined by taking the following finite difference ratio:

\begin{equation}
D_{p,q} f(x) \eqdef
\left\{
\begin{array}{ll}
\frac{f(px)-f(qx)}{(p-q)x},& x\not =0\mbox{\ and\ } p\not= q,\\
f'(0),& x=0 \mbox{\ or\ } p=q.
\end{array}
\right.
\end{equation}
We have $D_{p,q} f(x)=D_{q,p} f(x)$.

The $(p,q)$-derivative  is an extension of Jackson's {\em $q$-derivative}~\cite{Jackson-1909,qbook-2001,qbook-2012,pq-2013} historically introduced in 1909.
Notice that this finite difference differential operator that does not require to compute limits (a useful property for derivative-free optimization), and moreover can be applied even to
nondifferentiable (e.g., ReLU) or discontinuous functions.

An important property of the $(p,q)$-derivative is that it generalizes the ordinary derivative:
\begin{lemma}
For a twice continuously differentiable function $f$, 
we have $\lim_{p\rightarrow q} D_{p,q} f(x)= \frac{1}{q} D f(qx)=\frac{1}{q} f'(qx)$ and $\lim_{x\rightarrow 0} D_{p,q} f(x)=D f(0)=f'(0)$.
\end{lemma}

\begin{proof}
Let us write the first-order Taylor expansion of $f$ with exact Lagrange remainder for a twice continuously differentiable function $f$:
$$
f(qx)=f(px)+(qx-px)f'(px)+\frac{1}{2}(qx-px)^2f''(\varepsilon),
$$
for $\varepsilon\in (\min\{px,qx\}, \max\{px,qx\})$.

It follows that
\begin{eqnarray}
D_{p,q} f(x) &=& \frac{f(px)-f(qx)}{x(p-q)} = \frac{f(qx)-f(px)}{x(q-p)},\\
&=& f'(px)+\frac{1}{2}x(q-p)f''(\varepsilon).
\end{eqnarray}

Thus, whenever $p=q$ we have $D_{p,q} f(x)=f'(px)= \frac{1}{p} D f(px)$, 
and whenever $x=0$, we have $D_{p,q} f(0)=f'(0)$.
In particular, when $p=1$, we have $D_{q} f(x)=f'(x)$ when $q=1$ or when $x=0$. 
\end{proof}

Let us denote $D_q$ the $q$-differential operator $D_q \eqdef D_{1,q}=D_{q,1}$.

Since $B_F(qx:px)\eqdef f(qx)-f(px)-(qx-px)f'(px)=\frac{1}{2}(qx-px)^2f''(\varepsilon)$, 
we can further express the $(p,q)$-differential operator using Bregman divergences~\cite{BD-2005} as follows:

\begin{corollary}
We have:
\begin{eqnarray*}
D_{p,q} f(x)&=&\frac{f(px)-f(qx)}{x(p-q)}=f'(px)+\frac{B_F(qx:px)}{x(p-q)},\\
&=& \frac{f(qx)-f(px)}{x(q-p)}=f'(qx)+\frac{B_F(px:qx)}{x(q-p)}.
\end{eqnarray*}
\end{corollary}

\subsection{Leibniz $(p,q)$-rules of differentiation}
The following $(p,q)$-Leibniz rules hold:

\begin{itemize}
	\item Sum rule (linear operator):
	$$
	D_{p,q}(f(x)+\lambda g(x)) = D_{p,q} f(x) +  \lambda D_{p,q}g(x)
	$$
	
	\item Product rule:
	\begin{eqnarray*}
	 D_{p,q}(f(x)g(x))&=& f(px)+D_{p,q} g(x)+ g(qx)D_{p,q} f(x),\\
	&=& f(qx)+D_{p,q} g(x)+ g(px)D_{p,q} f(x).
	\end{eqnarray*}
	
	\item Ratio rule:
		\begin{eqnarray*}
	 D_{p,q}(f(x)/g(x))&=& \frac{g(qx)D_{p,q} f(x) - f(qx)D_{p,q} g(x)}{g(px)g(qx)},\\
	&=&  \frac{g(px)D_{p,q} f(x) - f(px)D_{p,q} g(x)}{g(px)g(qx)}.
	\end{eqnarray*}
\end{itemize}

\subsection{The $(p,q)$-gradient operator}

For a multivariate function $F(x)=F(x_1,\ldots, x_d)$ with $x=(x_1,\ldots,x_d)$, 
let us define the {\em first-order partial derivatives} for $i\in [d]$
and $p_i\not =q_i$,
\begin{eqnarray*}
D_{p,q,x_i} F(x) &\eqdef &
\left\{
\begin{array}{ll}
\frac{F(x_1,\ldots,p_i x_i,\ldots, x_n)-F(x_1,\ldots,q_i x_i,\ldots, x_n)}{(p_i-q_i)x_i} & x_i\not =0,\\
\frac{\partial F(x)}{\partial x_i} & x_i=0
\end{array}
\right.\\
&= &
\left\{
\begin{array}{ll}
\frac{F(x+(p_i-1)e_i)-F(x+(q_i-1)e_i)}{(p_i-q_i)x_i} & x_i\not =0,\\
\frac{\partial F(x)}{\partial x_i} & x_i=0
\end{array}
\right.,
\end{eqnarray*}
where $e_i$ is a one-hot vector with the $i$-th coordinate at one, and all other coordinates at zero.

The generalization of the $(p,q)$-gradient~\cite{qgrad-2013} follows by taking $d$-dimensional vectors for $p=(p_1,\ldots,p_d)$ and $q=(q_1,\ldots,q_d)$:

$$
\nabla_{p,q} F(x):=\left[
\begin{array}{c}
D_{p_1,q_1,x_1} F(x)\\
\vdots\\
D_{p_d,q_d,x_d} F(x)
\end{array}
\right]
$$

The $(p,q)$-gradient is a linear operator: $\nabla_{p,q} (aF(x)+bG(x))=a \nabla_{p,q}  F(x)+  b \nabla_{p,q}  G(x)$ for any constants $a$ and $b$.
When $p,q\rightarrow 1$, $\nabla_{p,q}\rightarrow \nabla$: That is, the $(p,q)$-gradient operator extends the ordinary gradient operator.

\subsection{Code snippet in Python}\label{sec:pseudocode}

We can easily implement $q$-neurons based on the following reference code,
which is based on a given activation function \texttt{activate}.
Note, \texttt{q} has the same shape as \texttt{x}.
One can fix \texttt{eps}$=10^{-3}$ and
only has to tune the hyper-parameter \texttt{lambda}.

\begin{verbatim}
def qactivate( x, lambda, eps ):
    q = random_normal( shape=shape(x) )
    q = ( 2*( q>=0 )-1 ) * ( lambda * abs(q) + eps )
    return ( activate( x * (1+q) ) - activate( x ) ) / q
\end{verbatim}

\end{document}